\documentclass[12pt]{article} 

\usepackage{times}
\usepackage[a4paper,
            bindingoffset=0.2in,
            left=1in,
            right=1in,
            top=1in,
            bottom=1in,
            footskip=.25in]{geometry}

\newif\ifpaper
\papertrue 

\usepackage{natbib}

\usepackage{lineno}

\usepackage{pgfplots}

\usepackage{amsmath}
\usepackage{amssymb}
\usepackage{natbib}
\usepackage{graphicx}
\usepackage{url} 
\usepackage{algorithm2e}
\usepackage{fancyvrb,moreverb,varwidth}
\usepackage{amsfonts}       
\usepackage[utf8]{inputenc} 
\usepackage[T1]{fontenc}    
\usepackage{hyperref}       
\usepackage{caption}
\usepackage{subcaption}
\usepackage{amsthm}
\usepackage{amsfonts}
\usepackage{pgfplots}

\usepackage{booktabs} 

\newtheorem{definition}{Definition}
\newtheorem{theorem}{Theorem}
\newtheorem{lemma}{Lemma}

\DeclareMathOperator*{\argmin}{arg\,min}

\newcommand{\calM}{\mathcal{M}}

\newcommand{\bbR}{\mathbb{R}}
\newcommand{\bbN}{\mathbb{N}}

\newcommand{\relu}{\textnormal{ReLU}}
\newcommand{\textwhere}{\textnormal{\ where\ }}
\newcommand{\textand}{\textnormal{\ and\ }}

\newcommand{\vol}{\textnormal{vol}}

\newcommand{\err}{\textnormal{err}}
\newcommand{\dif}{\textnormal{d}}

\newenvironment{hproof}{%
  \proof}{\endproof}

\author{%
 Kyle Matoba \url{kyle.matoba@epfl.ch}\\
 Idiap Research Institute and EPFL
 \and
 Nikolaos Dimitriadis \url{nikolaos.dimitriadis@epfl.ch}\\
 EPFL
 \and
 Fran\c{c}ois Fleuret \url{francois.fleuret@unige.ch}\\
 University of Geneva%
}

\title{The Theoretical Expressiveness of Maxpooling}
\begin{document}

\maketitle

\begin{abstract}
Over the decade since deep neural networks became state of the art image classifiers there has been a tendency towards less use of max pooling: the function that takes the largest of nearby pixels in an image. Since max pooling featured prominently in earlier generations of image classifiers, we wish to understand this trend, and whether it is justified. We develop a theoretical framework analyzing ReLU based approximations to max pooling, and prove a sense in which max pooling cannot be efficiently replicated using ReLU activations. We analyze the error of a class of optimal approximations, and find that whilst the error can be made exponentially small in the kernel size, doing so requires an exponentially complex approximation.

Our work gives a theoretical basis for understanding the trend away from max pooling in newer architectures. We conclude that the main cause of a difference between max pooling and an optimal approximation, a prevalent large difference between the max and other values within pools, can be overcome with other architectural decisions, or is not prevalent in natural images.

\end{abstract}


\section{Introduction}
When convolutional neural networks first became state of the art image classifiers, max pooling was a fundamental aspect of modelling. Some recent studies have argued that max pooling operations are unnecessary -- for instance because strided convolutions composed with ReLU nonlinearity can achieve the same outcome more simply and flexibly (\cite{Springenberg2014}). 
And practice has largely followed this observation -- whereas VGG (\cite{Simonyan_Zisserman_2015}) and AlexNet (\cite{Krizhevsky_Sutskever_Hinton_2017}) had several max pooling layers, many modern ResNets (\cite{He2016}) feature only a single max pooling layer, and some other important image classifies -- such as InceptionV3 (\cite{Szegedy_Vanhoucke_Ioffe_Shlens_Wojna_2016}) and mobilenetV3 (\cite{Howard_Pang_Adam_Le_Sandler_Chen_Wang_Chen_Tan_Chu_et_2019}) -- have none at all, despite being much deeper overall.\footnote{All statements about historical models refer to their reference implementation in torchvision, described here: \url{https://pytorch.org/vision/stable/models.html}.} And it would be lovely if max pooling could be dropped from the toolbox of convolutional neural network-based image classifiers, because it would ease the practical details of designing architectures, and better focus theoretical and development efforts. 

In this paper, we examine whether this is possible. We show that for some inputs, max pooling can give very different outputs than an optimal approximation by ReLUs. We derive comprehensive bounds on the error one realises by trying to approximate max functions with the composition of ReLU and linear operations, and find that a simple divide and conquer algorithm that progressively computes pairwise maxes (which requires $\log_2(d)$ applications for the maximum of $d$ inputs) cannot be improved upon, but that quite accurate approximations can be built, though they require $O(2^d)$ computation. We show that a natural approximation of modest complexity has surprisingly high error, and we present a convex optimization problem that characterizes the error from all intermediate approximations. 

This result does not say that it is wrong to omit max pooling from newer image classifiers. Rather, it says when it could be wrong. Even if max pooling is not empirically necessary to achieve SOTA accuracy, understanding more precisely what inductive biases one imposes by its omission is useful. Our results show that, max pooling is more efficient than an equivalent ReLU + linear scheme for inputs that have very high ranges of values within pools, so if we see that on natural images, max pooling is not needed, then we learn that natural images tend not to be highly dispersed. 
Furthermore, whilst the two operations might be of comparable accuracy, one may perform more reasonably when confronted with unnatural images, and we conduct a short experiment that preliminarily shows that networks built with max-pooling are more adversarially robust than an approximation by ReLUs.

\subsection{Notation}


$A^\top$ indicates the transpose of a matrix $A$, and $A^\dagger$ its Moore-Penrose pseudoinverse. 
$I_d$ denotes the identity matrix in $\bbR^{d \times d}$ and $1_d$ denotes a column vector of ones in $\bbR^d$. $e_{dj} \in \bbR^d$ denotes the $j$th column of $I_d$. $\Delta_d \subset \bbR^{d+1}$ denotes the standard $d$-simplex. We use the \emph{order statistics} notation in which $x_{(i)}$ (the subscripts being enclosed in parentheses) denotes the $i$th largest element of a vector $x = (x_1, x_2, \hdots, x_d)$. $\chi_A$ denotes the indicator function of the set $A$: $\chi_A(x) = 1$ if $x \in A$, and zero otherwise. 

This paper amounts to a careful analysis of the maximum function in $d$ dimensions, but we will also need to use the standard notion of the largest vector. We use the upper-case $\textsc{max}$ function to denote the function $\bbR^d \rightarrow \bbR$ and the notation $\max_i x_i$ to be the largest value of the vector $x$. We use the phrase ``order'' to indicate the size of the argument to a function.

We use the phrase ``ReLU network'' to mean a deep neural network comprised of alternating linear and ReLU layers, and ``ReLU block'' to mean a ReLU network that is part of a larger network. Here a linear layer is used expansively to include batch normalization, convolution, average pooling, etc. and compositions thereof.

\section{The complexity of max pooling operations}
\label{sec:width_effect}

In this section, we prove that in a simplified model, max pooling \emph{requires} depth -- multiple layers of ReLU nonlinearity are necessary in order to effect the same computation, and more layers are needed for larger windows. 

Max pooling is generally motivated as a way of aggregating over values in order to summarize and reduce dimensionality. However, this interpretation can be totally separated from its mathematical analysis, and doing so is more general and also simpler. Thus, throughout this paper, we examine the problem of approximating a max function which puts aside much of the unnecessary complexity around ``pooling'' specific considerations like striding, padding, and dilation, which ultimately amount to extracting the arguments to a max operator. In this sense, our model of max pooling can be seen as a simplifed version of the fittable \emph{maxout} activation proposed by \cite{Goodfellow2013} that fixes the linear mapping applied prior to the application of a maximum function to be the identity. 


\subsection{The use of max pooling in deep learning}

A standard measure of the complexity of a neural network with piecewise linear nonlinearities is how many distinct linear regions it divides the domain into (\cite{Hanin2019, Arora2018}). A linear aggregation followed by a ReLU clearly takes on two distinct linear functions, whilst the max of $d$ variables can, in general, behave differently on each of the $O(2^d)$ subsets where each of its arguments is equal to the maximum. Nonetheless, \autoref{theorem:equivalence} demonstrates that evaluating the heaviside function of a max function using a small, one-layer ReLU network is possible.

\begin{theorem}
\label{theorem:equivalence}
There exists a ReLU network with $d$ inputs, $d$ hidden neurons, and one output, $f$ such that for all $\xi \in \bbR$, $f(x - \xi) \le 0 \iff \max \{x_1, \hdots, x_d\} \le \xi$.
\end{theorem}

\begin{proof}
\begin{align*}
\max \{x_1, \hdots, x_d\} \le \xi \iff x_1 \le \xi \textand \hdots \textand x_d \le \xi \iff \sum_{k=1}^d \relu(x_k - \xi) \le 0.
\end{align*}
\end{proof}

This means that, as it pertains to classification accuracy in a network with a single source of nonlinearity, max pooling \emph{can} be replaced with a ReLU layer. However, since max pooling is typically used to construct hidden layers values (``C-cells'' in the terminology of early work drawing direct analogies to the visual cortex, such as \cite{Fukushima1980}), and not the direct computation of final logits, it does not imply that max pooling can be replaced with ReLUs throughout a deep neural network. 

\cite{Telgarsky2016} showed that deep neural networks cannot be concisely simulated by shallow networks. They did this by demonstrating a classification problem that is easy for deep networks to solve, but is provably difficult for shallow networks. We seek to do similarly -- building a test problem on which max pooling succeeds and ReLU fails, however \autoref{theorem:equivalence} shows that classification accuracy alone is not the correct notion of approximation. Rather showing that it is difficult for ReLU blocks to compute the maximum function with a bound on the $L_\infty$ error. Since we are interested in the indispensability of max pooling to build intermediate features, this is the most reasonable metric. 


\subsection{Computing max using ReLU}
In two dimensions, $\max(a, b) = (\relu(a - b) + \relu(b - a) + a + b) /2$, however this simple corresponence breaks down for $d > 2$. \autoref{sec:quintic_and_above} gives an interesting heuristic argument for the same fact using an approach wholly separate to ours. \autoref{thm:min_approx_error} is a scaled-up formulation of this fact, showing how a deep neural network can be used to recursively form pairwise maxes in order to compute the maximum of many variables in the obvious way. %

\begin{lemma}
\label{thm:min_approx_error}
There exists a $\lceil \log_2(d) \rceil$-hidden layer ReLU block with $k$th hidden layer size $2^{\lfloor(\log_2(d - 1) + 1) \rfloor} / 2^{k - 1}$ that evaluates $\textsc{max}: \bbR^d \rightarrow \bbR$.
\end{lemma}

The idea is that a ReLU function can straightforwardly compute the max of two values, so after $\lceil \log_2(d)\rceil$ iterations of pairwise maxima can we compute the maximum of $d$ variables. For example, 

\begin{align*}
& \max(x_1, x_2, x_3, x_4, x_5) = \max(z_1, z_2) \\
& \textwhere z_1 = \max(z_3, z_4), z_2 = \max(z_5, z_6)  \\
& \textwhere z_3 = \max(x_1, x_2), z_4 = \max(x_2, x_3), z_5 = \max(x_3, x_4), z_6 = \max(x_4, x_5).
\end{align*}

We use this construction to form exact maxes as approximations to maxes of larger order subsequently.

\autoref{thm:min_approx_error} is a precise upper bound on the width and depth necessary to evaluate a max function. A converse to this lemma -- that $\textsc{max}$ cannot be written as a smaller network -- is the main contribution of this work and is proven in \autoref{thm:main_theorem_width}. Note however that our results are restricted to ReLU networks with weight patterns that are constrained to essentially form pairwise maxes -- we do not allow weights to move freely. This is reasonable given that universal approximation theorems (e.g. \cite{Hornik1991, Cybenko1989}) guarantee that wide neural networks can evaluate any function, thus any practical analysis of max pooling needs to be made at a constrained width. 

In \autoref{sec:family_of_approximations} we present some abstract analysis of the max function, divorced from an interpretation as a deep learning model, and in \autoref{sec:complexity} we return to its implementation as a DNN.

\section{A family of approximations to the max function}
\label{sec:family_of_approximations}

\subsection{Subpool maxes}
\label{sec:subpool_maxes}
A subpool max is an operation that takes the maximum over a subset of the components of a vector. Formally, for a vector $x \in \bbR^d$ and some subset $R$ of $\{1, \hdots, d\}$ let $s(x; R) = \max\{x_j: j \in R\}$. For example, if $x = (3, 2, 10, 5)$, then $s(x; \{1, 2, 4\}) = \max(3, 2, 5) = 5$. Subpool maxes are an interesting generalization of the max function that present a natural tradeoff between complexity and accuracy. Let $C(k, r, d)$ denote the $k$th of subset of $\{1, \hdots, d\}$ of size $r$ (in the lexicographic ordering, without loss of generality). For example, $C(1,2,3)=\{1,2\}, C(2,2,3)=\{1,3\}$ and $C(3,2,3)=\{2,3\}$.  
For $R \subseteq \{0, 1, \hdots, d-1\}$ let $\calM_d(R)$ denote the space of affine combinations of subpool maxes of $r \in R$ variables, 

\begin{align}
\label{eqn:affine_combinations}
\calM_d(R) = \left\{ x \mapsto \beta_0 + \sum_{r \in R \backslash \{0\}} \sum_{j=1}^{{d \choose r}} \beta_r^j s(x; C(j, r, d)) : \beta_r^j, \beta_0 \in \bbR \right\},
\end{align}

if $0 \in R$, if $0 \not\in R$, then the intercept $\beta_0$ is omitted. \autoref{thm:main_theorem_width} shows that $\textsc{max} \not \in \calM_d(\{0, 1, \hdots, d - 1\})$, with a bound on the $L_\infty$ error from this function class. Before presenting the theorem, we start with a simple observation that is interesting in its own right.
 
\begin{lemma}
\label{thm:subpool_max_sums}
Let $S(x; r, d) = \frac{1}{{d \choose r}} \sum_{j=1}^{{d \choose r}} s(x; C(j, r, d))$ be the average of all ${d \choose r}$ subpool maxes of $x \in \bbR^d$ of order $r$. 
Then

\begin{align}
\label{eqn:Sxrd}
S(x; r, d) = \frac{1}{{d \choose r}} \sum_{j = 1}^{d - r + 1} {d - j \choose r - 1}x_{(j)}.
\end{align}
\end{lemma}

\begin{proof}
The $j$th largest value in $x$, $x_{(j)}$, will be the largest value within a subpool if and only if all indices $k < j$ are excluded from that subpool, and $j$ is not excluded. For a subpool of size $r$ the $r-1$ remaining values -- $x_{(j+1)}, \hdots, x_{(d)}$ -- must be chosen from amongst the $d-j$ indices not less than $j$. 
\end{proof}

When $r = 1$ in \autoref{eqn:Sxrd}, the expression reduces to a simple average, which is only slightly informative about the maximum, but when $r = d-1$, $S(x; d - 1, d) = ((d - 1) x_{(1)} + x_{(2)}) / d$, making it a quite reasonable approximation to $x_{(1)}$. For example at the point $x$ which is all zeros except for a single coordinate of 1 the average will have an error of $1 - 1 / d$, whilst the average of maxes over subpools of size $d-1$ will have an error of $1 / d$. 

In essence, the average of subpool maxes of a certain order $r \in R$ give a summary of the quantiles of the distribution via a particular weighted average, with better fidelity to the max for larger $r$. The idea is demonstrated in \autoref{fig:statistical_sorting}.

\begin{figure}
  \centering
  \ifpaper
    \input{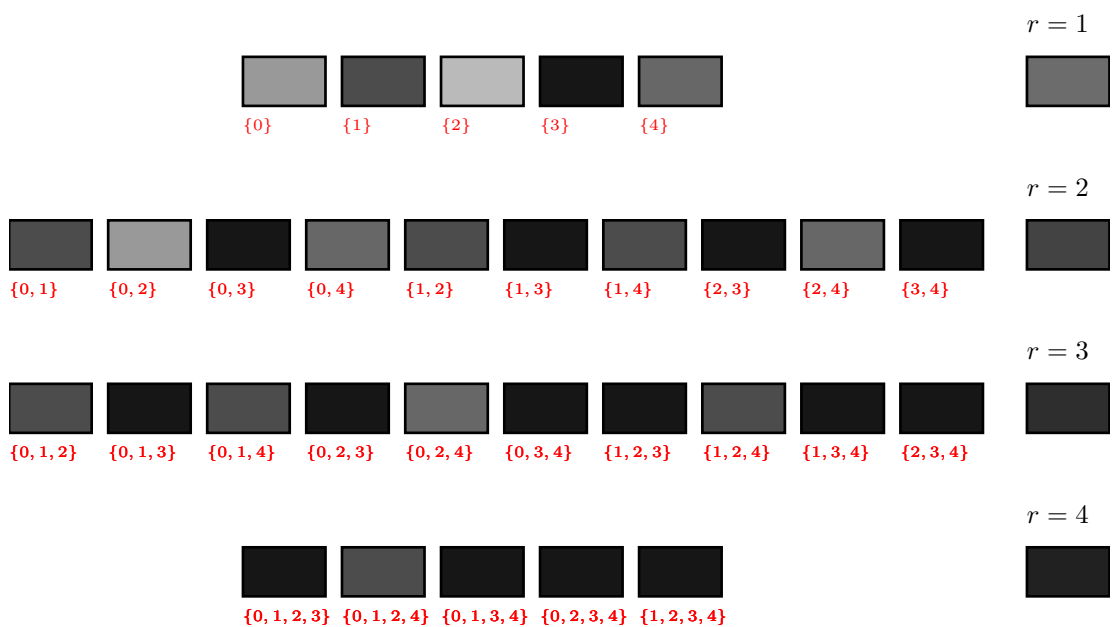}
  \fi
  \caption{A demonstration of the averaged subpool maxes for $d = 5$. Here the numerical value is presented as darkness in an inverted grayscale. Each row $r = 1, 2, 3, 4$ indicates the averaged subpool maxes over pools of order $r$, the included terms in the max are in red below. As the order of the subpool maxes grows, the averaged value (on the right) grows darker towards the actual max of the $d$ values.}
\label{fig:statistical_sorting}
\end{figure}

\subsection{A family of approximation bounds}
\autoref{eqn:Sxrd}, which expresses the average of subpool maxes as a linear combination of order statistics, is convenient for computation, and we are ready for an important theorem about the depth required to emulate a max pooling.

\begin{lemma}
An upper bound on the error of an $R$ estimator is given by the solution to this convex optimization problem, in $2(d + 1)$ constraints, and 1 + $|R|$ variables.

\begin{equation}
\begin{aligned}
\label{eqn:criterion}
&\min_{g, \beta_0, (\beta_r, r \in R \backslash \{0\})}\ g \textnormal{ subject to } \\
& |1 - \beta_0 - \sum_{r \in R \backslash \{0\}} \beta_r S((1, 1, 1, \hdots, 1, 1); r, d) | \le g, \\
& |1 - \beta_0 - \sum_{r \in R \backslash \{0\}} \beta_r S((1, 1, 1, \hdots, 1, 0); r, d) | \le g, \\
& |1 - \beta_0 - \sum_{r \in R \backslash \{0\}} \beta_r S((1, 1, 1, \hdots, 0, 0); r, d) | \le g, \\
& \hdots \\
& |1 - \beta_0 - \sum_{r \in R \backslash \{0\}} \beta_r S((1, 0, 0, \hdots, 0, 0); r, d) | \le g, \textnormal{ and } \\
&|\beta_0 + \sum_{r \in R \backslash \{0\}} \beta_r S(0; r, d) | \le g.
\end{aligned}
\end{equation}

\end{lemma}
\begin{proof}
Let $\calM_d^s(R)$ denote the restriction of $\calM_d(R)$ to functions that are symmetric in their arguments. We show in \autoref{lemma:permutation_invariance} that maximizing over this set suffices. For any set of points $P \subseteq [0, 1]^d$, it is clear that

\begin{align*}
\min_{m \in \calM_d^s(R)} || m - \textsc{max} ||_\infty \ge \min_{m \in \calM_d^s(R)} \max_{x \in P} |m(x) - \textsc{max}(x)|.
\end{align*}

We use this with 

\begin{align*}
P = \{ 
& (1, 1, 1, \hdots, 1, 1), \\
& (1, 1, 1, \hdots, 1, 0), \\
& (1, 1, 1, \hdots, 0, 0), \\
&\vdots \\
& (1, 0, 0, \hdots, 0, 0) \\
& (0, 0, 0, \hdots, 0, 0)\}.
\end{align*}

So that 
\begin{align*}
\min_{m \in \calM_d^s(R)} \max_{x \in P} |m(x) - \textsc{max}(x)|.
\end{align*}

And \autoref{eqn:criterion} follows from a standard trick for rewriting $L_\infty$ optimization (see, for example, \cite{Boyd2004}). 
\end{proof}

This is a convex optimization problem that can be efficiently solved by standard software. \autoref{sec:optimal_approximation_error} shows the results of this computation for small $d$. 

We solve \autoref{eqn:criterion} as functions of $d$ for three particularly interesting $R$ in the main theorem of this paper, \autoref{thm:main_theorem_width}. In the case of \autoref{eqn:d1} and \autoref{eqn:0d1} we further establish tightness of the bound. 

\begin{theorem}
\label{thm:main_theorem_width} 
Let $|| \cdot ||_\infty$ denote the $L_\infty$ norm of a function defined over the unit cube. Let $\err(R) = \min_{m \in \calM_d(R)} || m - \textsc{max} ||_\infty$. 

\begin{align}
\label{eqn:d1}
\err(\{d-1\}) &= 1 / (2d - 1)  \\
\label{eqn:0d1}
\err(\{0, d-1\}) &= 1 / (2d) \\
\label{eqn:012d1}
\err(\{0, 1, 2, \hdots, d-1\}) &\le 1 / 2^d.
\end{align}
\end{theorem}

\begin{hproof}
The idea of the proof of all three assertions is to assumes that the $L_\infty$ norm of the error at is characterized by a few key vertices. For \autoref{eqn:d1} and \autoref{eqn:0d1} these amount to the origin and one or two other points at which the max is one. 
Under this conjecture, the norm can be optimized by computed by evaluating the error at each point, recognizing the tension between them, and finding coefficients which equate them. Given the conjectured error, we can then prove it is optimal by contradiction. For \autoref{eqn:012d1}, the argument is a bit more involved, and we approach it using the vertex representation. Because of the greater complexity not attempt to prove the equality of the conjecture, but settle for an upper bound. 
\end{hproof}

Contrasting \autoref{eqn:d1} and \autoref{eqn:0d1} serves mainly to highlight the importance of the intercept. Since the intercept requires negligible computation, we assume its inclusion subsequently.

Bounding the $L_\infty$ error is a strong result, but one that requires a bit of interpretation, since it is possible for the $L_\infty$ error to be high only on a set of zero measure (and therefore possibly a technical detail that is not pertinent in applications). We show that this is not the case in \autoref{lemma:nonnegligible_measure}, which demonstrates lower bound on the measure of a set on which the $L_1$ norm is high. 
In the subsequent discussion, we call the optimal estimate based on the terms $R \subseteq \{0, 1, \hdots, d\}$ the $R$-estimate, and let $\beta^\star$ denote the optimal coefficients.

\begin{lemma}
\label{lemma:nonnegligible_measure}
For $\epsilon< \err(R) / 2$, let $W(\epsilon; R) = \{x \in [0, 1]^d : |x_{(1)} - \beta_0^\star - \beta^{\star \top} B(d) S(x;, r, d) | \ge \epsilon \}$ denote the subset of the unit cube where the error of an $R$-estimator is at least $\epsilon$. Then for all $R$ with $0 \in R$, $\vol(W(\epsilon; R)) \ge (\err(R) / 2 - \epsilon)^d$, where $\vol$ is the Lebesgue measure over $[0, 1]^d$. 
\end{lemma}

This bound could be improved in a number of ways discussed in \autoref{sec:proof_nonnegligible_measure}, however, it suffices to demonstrate that the error is high on a subset of the unit cube of nonnegligible measure. In \autoref{sec:main_theorem_width_L2} we solve for the $L_2$ error of \autoref{eqn:012d1} in closed form and find that it is also not zero. 

\autoref{eqn:012d1} forms our main result: even using all lower-order subpool maxes one cannot simulate the maximum of $d$ variables. Two main points emerge from this analysis: (1) the quality of the approximation can be quite good, with exponentially small error in $d$, and (2) introducing lower order subpool maxes into the approximation reduces approximation error, and not just a little bit. The first point is contrary to the usual ``curse of dimensionality'' -- the quality of the approximation improves in higher dimensions. This can be explained by the tremendous complexity of the approximation, with higher $d$, exponentially many more terms can be included in the approximation. Indeed, increasing the dimensionality, but allowing the number of the terms in the approximation to grow only proportionally only lowers the error like $1 / d$. 

Observe that all orders of subpool max contribute to the quality of an approximation: although the subpool maxes of order $d - 1$ are most informative about the max, all subpool max averages improve an estimate. Examining the approximations, the explanation is clear: including lower order maxes can offset loadings on lower order statistics so that an estimator can more precisely extract the loading on higher order statistics. Indeed, in the extreme case where all terms are included, there is an alternating pattern in the coefficients signs (see \autoref{eqn:lambda_weights}). And this effect is strong, the inclusion of all terms improves order of approximation error from $O(1/d)$ to $O(1 / 2^{d})$.


\section{The complexity of $\calM_d(R)$}
\label{sec:complexity}

How can we reconcile the very disparate orders of estimation error between \autoref{eqn:d1} and \autoref{eqn:012d1}? By noting that the number of intermediate calculations necessary to form an estimator in $\calM_d(\{0, 1, 2, \hdots, d-1\})$ is high, whilst the features needed to construct an element of $\calM_d(\{0, d-1\})$ can be limited. In order to show that the computation required to evaluate an element of $\calM_d(\{0, 1, 2, \hdots, d-1\})$ is considerable, we require some further concepts.

\begin{definition}[Implementing a subpool-max sum, implementing an $R$-estimator]
\label{def:implementation}
We say that a deep neural network \emph{implements} a subpool max of order $r$ if $s(x; C(j, r, d))$ is computed by neurons in the network for all $j = 1, \hdots, {d \choose r}$. We say that a DNN implements an $R$-estimator if it implements all orders $r \in R$.
\end{definition}

The idea is that following a forward pass through the network, all of the subpool maxes have been computed as nodes in the computational graph. 

In order for \autoref{def:implementation} to be meaningful, we need some assurance that in fact evaluating $S(x; r, d)$ is only possible by computing $s(x; C(j, r, d))$ for all $j$. Clearly, this is not true in general -- for instance, if sorting is a permitted operation then evaluating $S(x; r, d)$ is trivial. However, it is true under a natural model of being computed by a DNN, defined in \autoref{def:functional_decomposition}.

\begin{definition}[Functional decomposition of order $r$]
\label{def:functional_decomposition}
We say that a function $f: \bbR^d \rightarrow \bbR$ has an order-$r$ functional decomposition if there exist functions $f_1, f_2, \hdots, f_{K}$ such that $f = \sum_k f_k$ and each $f_k$ is a function of only $r$ coordinates.

We say that a function $f: \bbR^d \rightarrow \bbR$ has an order-$r$ functional decomposition in normalized form if there exist functions $f_1, f_2, \hdots, f_{{d \choose r}}$ such that $f = \sum_k f_k$ and $f_k$ is a function of the coordinates in $C(k, r, d)$ alone. 

\end{definition}

(General) functional decompositions can be put into a normalized form in $O(K)$ time and space by looping over each element of the functional decomposition, summing all terms that are a function of the same index set, then permuting the indices. 

When it is apparent, we omit the ``order-$r$'' bit. \autoref{thm:main_lower_bound} shows that $S(x; r, d)$ does not have an order-$r$ functional decomposition of any size $K < {d \choose r}$, and thus entails $O({d \choose r})$ calculation. 


\begin{theorem}
\label{thm:main_lower_bound}
In any order-$r$ functional decomposition of $S(x; r, d)$, $f_1, \hdots, f_K$, $K \ge {d \choose r}$. Implementing a $\{0, 1, 2, \hdots, d - 1\}$ estimator entails $O(2^d)$ computation. 
\end{theorem}

We would like to be able to conclude from \autoref{thm:main_lower_bound} that evaluating $S(x; r, d)$ with a DNN requires $O(2^d)$ computation, however proving that general DNNs \emph{cannot} do something efficiently is difficult (\cite{Abbe2020} is one example). The situation is that networks of the form introduced in \autoref{sec:width_effect} are a subset of all DNNs, and also a subset of all functional decompositions. We have shown that the second set has no representation with less than $O(2^d)$ computation, but the relationship between the two supersets is not clear. We find \autoref{thm:main_lower_bound} to be a reassuring first step and have begun looking at the more general approximation problem in ongoing work. 



In contrast to the very accurate approximation computed above, an upper bound on the number of lower-order subpool maxes necessary to compute an optimal $\{0, d - 1\}$-estimator is demonstrated constructively in \autoref{thm:d1_width}. The idea is that because we finally seek to evaluate subpool maxes that differ in only a single coordinate, significant reuse of subpool maxes is possible. 

One simple case is when the dimension is a power of two plus one. 
Then every intermediate subpool will be a power of two and at the $m$th split, there will be $2^m$ subtuples that are components of more than one tuple. Because the tuples are always of an even length, it is possible to build half of them from these shared subtuples. Thus $d + 2^{m - 1} - 1$ neurons are needed at the $m$th layer of the network.


\begin{theorem}
\label{thm:d1_width}
For $d \in \bbN$ and $j \le \lceil \log_2(d - 1)\rceil$ let $\zeta(d, j) = \lceil (d - 1) / 2^j \rceil$ be the size of a tuple that started with size $d - 1$ and has been halved $j$ times (rounding up). For $d \ge 3$, a $\{0, d - 1\}$-estimate can be computed by a ReLU network of depth less than $\textsc{depth}(d) = \lceil \log_2(d - 1)\rceil$, where the $j$th layer has width $w(j, d)$, given by:

\begin{equation}
\begin{aligned}
\label{eqn:width_condition}
w(d, 1) &= 2^{\lfloor \log_2(d - 2)\rfloor} + (d - 1) \\
w(d, j) &= 2^{\textsc{depth}(d) - j} \times (1 + \zeta(d, \textsc{depth}(d) - j)) \textnormal{ for } j = 2, 3, \hdots, \textsc{depth}(d).
\end{aligned}
\end{equation}
\end{theorem}


To make the idea concrete \autoref{fig:layer_bounds} demonstrates this computation in a tabular form for $d =  10$. $d = 10$ is an especially complicated dimensionality, since all $\zeta$ are odd. 

\begin{figure}[!ht]
  \begin{subfigure}[a]{\linewidth}
    \begin{Verbatim}[commandchars=\\\{\}, fontsize=\footnotesize]
    \underline{0  1  2  3  4}    | 4  5  6  7  8     \\  
    \underline{0  1  2  3  4}    | 4  5  6  7     9  \\
    \underline{0  1  2  3  4}    | 4  5  6     8  9  \\
    \underline{0  1  2  3  4}    | 4  5     7  8  9  \\
    \underline{0  1  2  3  4}    | 4     6  7  8  9  \\
    0  1  2  3     5 |    \underline{5  6  7  8  9} \\
    0  1  2     4  5 |    \underline{5  6  7  8  9} \\
    0  1     3  4  5 |    \underline{5  6  7  8  9} \\
    0     2  3  4  5 |    \underline{5  6  7  8  9} \\
       1  2  3  4  5 |    \underline{5  6  7  8  9} \\
    \end{Verbatim}
    \caption{$j = 3, \zeta(10, 3) = 5$: 2 repeated + 10 unique = $2^{4 - 3} \times (5 + 1)$.}
  \end{subfigure}
  \hfill
  \begin{subfigure}[b]{\linewidth}
    \begin{Verbatim}[commandchars=\\\{\}, fontsize=\footnotesize]
    \underline{0  1  2}    | \underline{2  3  4}    | \underline{4  5  6}    | 6  7  8     \\  
    \underline{0  1  2}    | \underline{2  3  4}    | \underline{4  5  6}    | 6  7     9  \\
    \underline{0  1  2}    | \underline{2  3  4}    | \underline{4  5  6}    | 6     8  9  \\
    \underline{0  1  2}    | \underline{2  3  4}    | 4  5     7 |    \underline{7  8  9} \\
    \underline{0  1  2}    | \underline{2  3  4}    | 4     6  7 |    \underline{7  8  9} \\
    \underline{0  1  2}    | 2  3     5 |    \underline{5  6  7} |    \underline{7  8  9} \\
    \underline{0  1  2}    | 2     4  5 |    \underline{5  6  7} |    \underline{7  8  9} \\
    0  1     3 |    \underline{3  4  5} |    \underline{5  6  7} |    \underline{7  8  9} \\
    0     2  3 |    \underline{3  4  5} |    \underline{5  6  7} |    \underline{7  8  9} \\
       1  2  3 |    \underline{3  4  5} |    \underline{5  6  7} |    \underline{7  8  9} \\
    \end{Verbatim}
    \caption{$j = 2, \zeta(10, 2) = 3$: 6 repeated + 10 unique = $2^{4 - 2} \times (3 + 1)$.}
  \end{subfigure}
  \begin{subfigure}[c]{\linewidth}
    \begin{Verbatim}[commandchars=\\\{\}, fontsize=\footnotesize]
    \underline{0  1}    | \underline{1  2}    | \underline{2  3}    | \underline{3  4}    | \underline{4  5}    | \underline{5  6}    | \underline{6  7}    | \underline{7  8}  \\   
    \underline{0  1}    | \underline{1  2}    | \underline{2  3}    | \underline{3  4}    | \underline{4  5}    | \underline{5  6}    | \underline{6  7}    | 7     9  \\   
    \underline{0  1}    | \underline{1  2}    | \underline{2  3}    | \underline{3  4}    | \underline{4  5}    | \underline{5  6}    | 6     8 |    \underline{8  9}   \\   
    \underline{0  1}    | \underline{1  2}    | \underline{2  3}    | \underline{3  4}    | \underline{4  5}    | 5     7 |    \underline{7  8} |    \underline{8  9}   \\   
    \underline{0  1}    | \underline{1  2}    | \underline{2  3}    | \underline{3  4}    | 4     6 |    \underline{6  7} |    \underline{7  8} |    \underline{8  9}   \\   
    \underline{0  1}    | \underline{1  2}    | \underline{2  3}    | 3     5 |    \underline{5  6} |    \underline{6  7} |    \underline{7  8} |    \underline{8  9}   \\    
    \underline{0  1}    | \underline{1  2}    | 2     4 |    \underline{4  5} |    \underline{5  6} |    \underline{6  7} |    \underline{7  8} |    \underline{8  9}   \\    
    \underline{0  1}    | 1     3 |    \underline{3  4} |    \underline{4  5} |    \underline{5  6} |    \underline{6  7} |    \underline{7  8} |    \underline{8  9}   \\     
    0     2 |    \underline{2  3} |    \underline{3  4} |    \underline{4  5} |    \underline{5  6} |    \underline{6  7} |    \underline{7  8} |    \underline{8  9}   \\    
       \underline{1  2} |    \underline{2  3} |    \underline{3  4} |    \underline{4  5} |    \underline{5  6} |    \underline{6  7} |    \underline{7  8} |    \underline{8  9}   \\      
    \end{Verbatim}
  \caption{$j = 1$: 9 repeated + 8 unique = $2^{\log_2 8 } + (10 - 1)$.}
  \end{subfigure}
 \caption{\autoref{thm:d1_width} demonstrated for $d = 10$ (zero-indexed for brevity), for which $\textsc{depth}(d) = 4$. Elements with no commonalities are not underlined. Split points are indicated by vertical lines. 
}
  \label{fig:layer_bounds}
\end{figure}

Note that that this need not be the smallest possible representation of the network, just an upper bound that is not too complex. A simpler, uniform bound is obtained by using the width of the the first layer, which itself has a very simple bound of width $\le 2d - 3$, showing that a $\{0, d-1\}$ estimate is of $O(d \log d)$ space and time complexity. 





So we reach a nuanced but reasonable conclusion: a $\{0, d - 1\}$-estimator gives a principled approximation with bounded error efficiently, and it is possible to improve upon it by including more and higher-order terms into $R$, all the way to an extremely good approximation, albeit at intractable levels of calculation. As an example, in $d = 9$, a $\{0, 8\}$-estimate can be implemented with a network of two hidden layers of size $12$ and $10$ and gives an error of 0.0556, whereas the same scheme applied to a $\{0, 1, 2, 3, 4, 5, 6, 7, 8\}$ estimate requires hidden layers of size $36, 84, 126, 126, 84, 36$ and gives an error of 0.0020. As an indication of the relative complexity of these two networks, we might use the bound on the number of linear regions from \cite{Serra2018}, which gives a value of 819906560 for the former estimator and $5.151431489308274 \times 10^{68}$ for the latter. 



\section{Analysis}
\label{sec:practice}

So far, we have strictly looked at the problem of approximating a max-function in isolation. In this short section, we examine briefly the practical takeaways along two dimensions: generalizability and adversarial robustness. 


\subsection{Generalizability}

Even if refit from scratch, \autoref{thm:main_lower_bound} shows that ReLU-based approximations cannot be a drop-in replacement for max pooling. By itself, however, this says nothing about whether maxpooling is empirically good for accuracy. 
And indeed, since we know that newer image classifiers tend to forego max pooling, we already have convincing evidence that excluding maxpooling need not be bad for accuracy. One useful, but incomplete, approach is to show that not enough max-pooling is surely \emph{bad} for accuracy. 

\autoref{thm:all_my_work} shows that the weight on all order statistics is positive, thus to oversimplify matters, an optimal estimator will act like a linear combination of the max and average of the pool. Since we well understand that average pooling is a poor substitute for max-pooling, we anticipate that a large enough step away from max pooling towards average pooling would be detrimental to performance. One way to interpret this reasoning is from a generalizability perspective: it may be safer to err on the side of less linear activations, that is, max-pooling over a linear + ReLU approximation to it, even if in-sample performance does not alone justify it.

Some interesting empirical work that bears directly on the importance of max pooling comes from \cite{Grunning2022}, who find that \emph{min}-pooling also performs well. This result is likewise rationalized by viewing max, average, and min pooling all as instances of general linear combinations of order statistics: the average of a pool performs poorly compared to more nonlinear pooling methods. 

\subsection{Adversarial robustness}
\newcommand{\lenet}{LeNet}
\newcommand{\resnet}{ResNet}
\newcommand{\mnist}{MNIST}
\newcommand{\cifar}{CIFAR10}

In this section, we investigate the connection between adversarial robustness and max pooling. Our hypothesis is that max pooling can be more robust than strided convolution and ReLU nonlinearity, since genuine max pooling admits only a single direction along which features can change -- the max. ReLU, by contrast, can be moved with only low correlation changes, and a random perturbation will in general change the output. 

Our experiments corroborate this intuition; omitting max pooling results in lower robust accuracy in several different model classes, ranging from simple Convolutional Neural Networks to ResNets. Our experimental approach is to adversarially attack models with and without max pooling. We use the Fast Gradient Sign Method by \cite{fgsm}. Specifically, starting from a model incorporating max pool layers, we replace them strided convolution + ReLU. We examine four different models and report the robust accuracy in \autoref{fig:resnet-experiment-main-text} on the \cifar \ dataset. The Maxpool model is from \cite{Page2018_8} and includes four maxpool layers. The baselines make the following modifications: Conv-Small and Conv-Large replace the max pool layers with a convolutional one, with kernel size one and three, respectively. The Conv-Strided model uses a strided convolution in lieu of the max pool layer and the convolution that precedes it. More details on the models as well as experiments with LeNet architecture can be found in \autoref{sec:appendix:experiments} of the appendix.

\begin{figure}[t]
    \centering
    \includegraphics[width=0.95\textwidth]{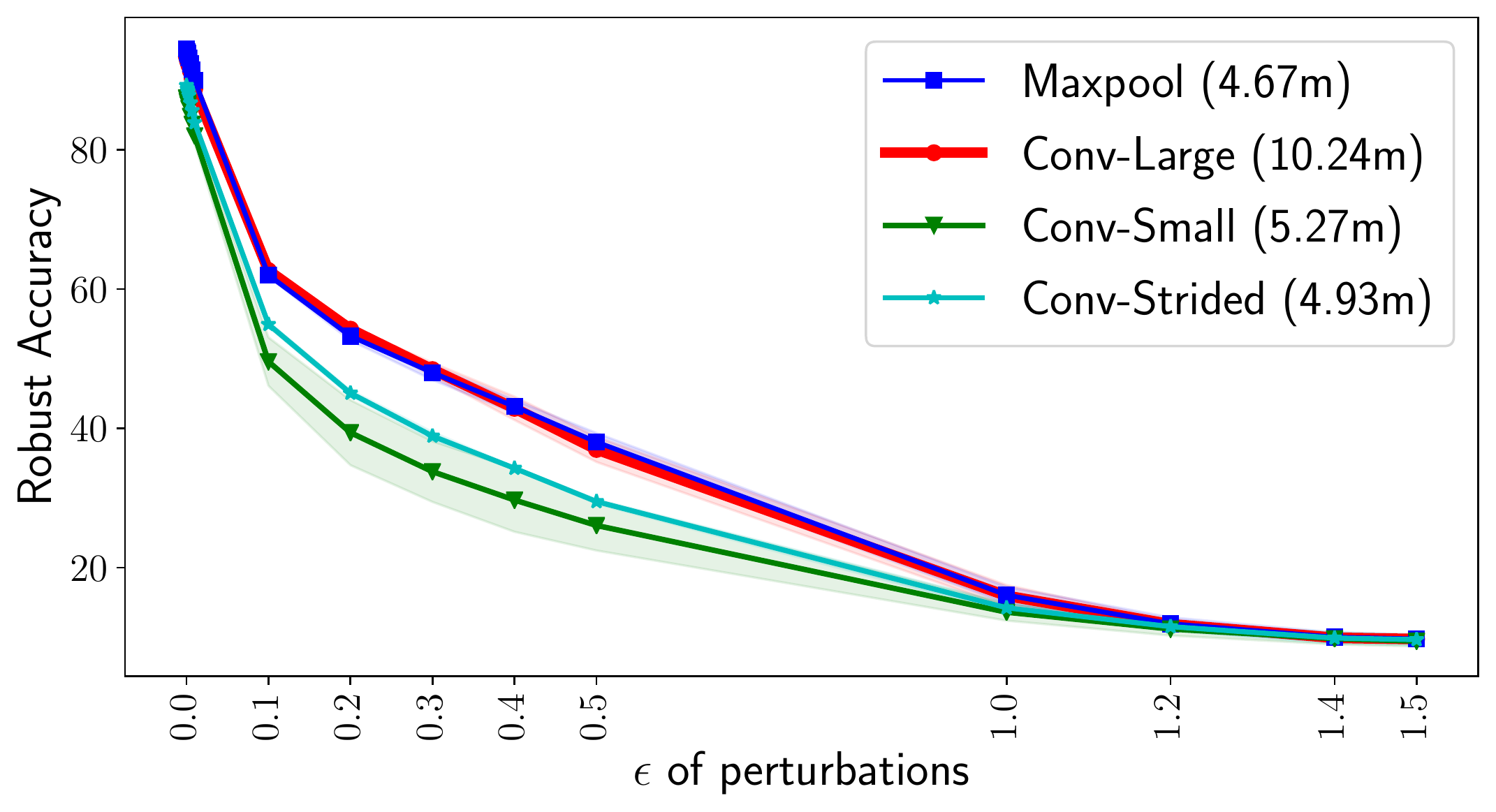}
    \caption{Effect of perturbations with \( L_\infty \leq \epsilon \) on the accuracy of the \resnet \ model variants on the \cifar\ dataset. Mean accuracy and standard deviation over the course of three runs are depicted. The legend indicates the number of parameters of each model in parentheses. The models omitting maxpool pay a price in either (robust) accuracy or model complexity (in terms of parameters).}
    \label{fig:resnet-experiment-main-text}
\end{figure}

The experimental results showcase a tradeoff between model complexity (in terms of number of parameters) and robust accuracy. The Maxpool model is more adversarially robust than the baselines Conv-Small and Conv-Strided. This effect is further highlighted for larger perturbations \( \epsilon \).  An exception to this trend lies in the Conv-Large model which is able to match the robust accuracy of the Maxpool model, but requires more than twice as many parameters.

\section{Conclusion}

We have posed and answered -- theoretically at least -- the question: can max pooling be replaced by linear mappings (such as strided convolutions) composed with ReLU activations, and when do we expect it to be considerably different? To do this, we first established the relevant parameters for comparison: distance in intermediate feature space is the correct notion. Next, we established a baseline: width needs to be constrained in order to address universal approximation theorems, and with a block of $\log_2(d)$ depth, a max pool kernel size of $d$ can be perfectly replicated with a simplistic divide and conquer algorithm. After positing the average of subpool maxes as a useful class of approximations, we gave our main impossibility theorem: with anything less than the na\"ive amount of depth, max pooling cannot be perfectly synthesized from ReLUs. Interestingly, we found that the error of an approximation encompasses the entire range between $O(1/d)$ in the dimension $d$ and $O(1/2^d)$ depending on exactly how much depth is used. To better understand this phenomenon, we analyzed the computational complexity of the approximations and reached an altogether tidy conclusion: to achieve exponentially little error, one needs exponentially much computation. Finally, we examined some practical implications of this analysis, concluding that whilst max pooling can be replaced in image classifiers if accuracy alone is the goal, max pooling may nonetheless be worthwhile if one wants more generally useful networks in the sense of adversarial robustness and generalizability. 

\section{Acknowledgements}
We thank Guillermo Ortiz-Jim\'{e}nez, Suraj Srinivas, Angelos Katharopoulos, and Arnaud Pannatier for helpful feedback. Kyle Matoba is supported by the Swiss National Science Foundation under grant number FNS-188758 ``CORTI''. The work of Nikolaos Dimitriadis was supported by Swisscom (Switzerland) AG.
\bibliography{biblio.bib}

\appendix

\section{In what sense is evaluating $\max(x_1, x_2, x_3, x_4, x_5)$ hard?}
\label{sec:quintic_and_above}

To motivate the fundamental differences of the max function with higher dimensionality, in this short section, we give a short argument for why there is unlikely to be a simple generalization of the correspondence between maximum and ReLU in $d = 2$ to general $d$. The idea is entirely due to \cite{Blatter2011}, though any mistakes in the concrete statement are ours alone. 

\begin{theorem}

There is no algebraic expression for 

\begin{align*}
(x_1, x_2, x_3, x_4, x_5) \mapsto \max(x_1, x_2, x_3, x_4, x_5).
\end{align*}
\end{theorem}

\begin{proof}
$\max(x_1, x_2, x_3, x_4, x_5)$ is a root of the polynomial $(x - x_1)(x - x_2)(x - x_3)(x - x_4)(x - x_5)$. Suppose that there was an algebraic expression for the largest of $(x_1, x_2, x_3, x_4, x_5)$, say $f(x_1, x_2, x_3, x_4, x_5)$ then, then the roots of $(x - x_1)(x - x_2)(x - x_3)(x - x_4)(x - x_5) / (x - f(x_1, x_2, x_3, x_4, x_5))$ could be found via the quartic equation, and we would have an algebraic expression for all five roots. However, Abel's Theorem states that there is no algebraic expression for general quintic polynomials.

\end{proof}
As mentioned, this technique was proposed by \cite{Blatter2011}. The argument is a bit subtle, so to better understand it, consider the same argument applied to the maximum of two values. The two roots of $(x - x_1)(x- x_2)$ are well known to be $(x_1 + x_2) /2  \pm \sqrt{(x_1 + x_2)^2 - 4x_1x_2} / 2$, and by inspection the larger corresponds to adding the discriminant:

\begin{align*}
\frac{(x_1 + x_2) + \sqrt{(x_1 + x_2)^2 - 4x_1x_2}}{2} = \frac{(x_1 + x_2) + \sqrt{(x_1 - x_2)^2}}{2} = \frac{(x_1 + x_2) + |x_1 - x_2|}{2}. 
\end{align*}

Which is a well-known trick for reasoning mathematically about the maximum of two variables. A corresponding expression comes from solving the cubic equation

\begin{align*}
\max(x_1, x_2, x_3) &= \frac{1}{2}\frac{x_1( |x_1 - x_2| + |x_1 - x_3|) + x_2(|x_1 - x_2| + |x_2 - x_3|) + x_3(|x_2 - x_3| + |x_1 - x_3|)}{|x_1 - x_2| + |x_2 - x_3| + |x_1 - x_3|} \\
& + \frac{|x_1 - x_2| + |x_2 - x_3| + |x_1 - x_3|}{4}.
\end{align*}

From examining the form above, we see that it is tractable because we can assess the min and max similar to above, and impute the third value from the average.\footnote{See also the excellent exposition given at \url{https://math.stackexchange.com/a/89702/92999}.} This gives a rudimentary version of \autoref{thm:subpool_max_sums}.

Presumably there is an even more complicated formula for $\max(x_1, x_2, x_3, x_4)$. However, for fifth and higher-order polynomials we cannot generally even write down an algebraic expression for the roots, much less determine by inspection which will be the greatest. 

\section{Proofs}
\label{sec:proofs}
\subsection{Proof of \autoref{thm:min_approx_error}}

\begin{proof}
For an input $x \in \bbR^k$, let $h_1(x) = (x_1, \hdots, x_{\lceil k / 2\rceil})$ and $h_2(x) = (x_{\lfloor k / 2 \rfloor + 1}, \hdots, x_k)$ be functions that extract, respectively, the first and second halves of $x$.\footnote{Here and throughout, we split inputs at their midpoint, but this could be replaced by any function that extracts half of the input, say odd and even indices, so long as the corner cases were adapted concomitantly.} Let $H$ be the linear mapping $\bbR^{d} \mapsto \bbR^{4 \lceil d / 2 \rceil}$,

\begin{align*}
H(x) = \begin{pmatrix} 
+h_1(x) - h_2(x) \\
-h_1(x) + h_2(x) \\
+h_1(x) + h_2(x) \\
-h_1(x) - h_2(x)
\end{pmatrix}.
\end{align*}
So, letting $A(x) = ((x_1, \hdots, x_{k / 4}) + (x_{k / 4 +1}, \hdots, x_{k / 2}) + (x_{k / 2 + 1}, \hdots, x_{3k / 4}) - (x_{3k / 4 + 1}, \hdots, x_k)) / 2$ from dimension $d$ to $d / 4$, $\max(h_1(x), h_2(x)) = A(\relu(H(x))$ is the composition of linear and ReLU functions. With these two linear mappings, we have that 

\begin{align*}
\textsc{max} = A \circ \relu \circ H \circ A \circ \relu \circ H \circ A \circ \hdots \circ H \circ A \circ \relu \circ H \circ P
\end{align*}

where $P$ is a function that pads the input size up to $2^{\lceil \log_2(d - 1) \rceil}$, say by appending $-\infty$. 
\end{proof}

\subsection{Proof of \autoref{thm:main_theorem_width}}
\begin{proof}

Let $\textnormal{perm}(d)$ denote the set of all permutations of $\{1, 2, \hdots, d\}$, and let

\begin{align*}
\calM_d^{s}(R) = \{m \in \calM_d(R) : m(x_1, x_2, \hdots, x_d) = m(x_{\sigma_1}, x_{\sigma_2}, \hdots, x_{\sigma_d}) \textnormal{ for all } \sigma \in \textnormal{perm}(d) \}
\end{align*}

denote the restriction of $\calM_d$ to those elements that are invariant to a reordering of its arguments. Because $\textsc{max}$ is symmetric in this sense, any optimal approximation to it must lie in $\calM_d^s$: 

\begin{lemma}
\label{lemma:permutation_invariance}
For all $R$, 
\begin{align*}
\min_{m \in \calM_d(R)} || m - \textrm{max}_d ||_\infty = \min_{m \in \calM^s_d(R)} || m - \textrm{max}_d ||_\infty.
\end{align*}
\end{lemma}

\begin{proof}
Assume otherwise, that is:
\begin{align*}
\min_{m \in \calM_d(R)} || m - \textsc{max} ||_\infty < \min_{m \in \calM^s_d(R)} || m - \textsc{max} ||_\infty.
\end{align*}

Meaning that there is an $m$ that is not symmetric and an $x$ such that $m(x') < m^s(x)$ for all symmetric $m^s$ and all $x'$. In particular $m(x_\sigma) < m^s(x)$ for all permutations $x_\sigma$ of $x$. Thus

\begin{align*}
\frac{1}{d!} \sum_{\sigma \in \textnormal{perm}(x)} m(x_\sigma') < m^s(x).
\end{align*}

However, $x\mapsto \frac{1}{d!} \sum_{\sigma \in \textnormal{perm}(x)} m(x_\sigma') < m^s(x)$ is evidently symmetric, a contradiction. 

\end{proof}

Thus, it is without loss of generality to optimize over $\calM^s$ rather than $\calM$, and we turn to operationalizing the symmetry assumption in terms of the coefficients. 

For $r > 1$, $s(1_d - e_{dk}; C(j, r, d)) = 1$ for all $k =  1, 2, \hdots, d$ and all $j = 1, \hdots, {d \choose r}$ since there is only a single non-one value. Thus all terms of order greater than 1 are equal, and the sums are equal by symmetry, so we necessarily have that for all $f \in \calM_d^s$, $\beta_1^1 = \beta_1^2 = \hdots = \beta_1^d$. Call this single value $\beta_1$, and (for $0, 1 \in R$)

\begin{align*}
\calM_d^s(R) = \left\{ x \mapsto \beta_0 + \beta_1 S(x; 1, d) + \sum_{r \in R \backslash \{0, 1\}} \sum_{j=1}^{{d \choose r}} \beta_r^j s(x; C(j, r, d)) : \beta_r^j, \beta_0, \beta_1 \in \bbR \right\}.
\end{align*}

Repeating this process for pools consisting of entirely of 1, except for $2, 3, \hdots, d- 1$ zeros in turn implies that the estimator must be a function of $S(x; r, d)$ alone, and not the individual terms of the sum separately:

\begin{align*}
\calM_d^s(R) = \left\{ x \mapsto \beta_0 + \sum_{r \in R \backslash \{0\}} \beta_r S(x; r, d) : \beta_r, \beta_0 \in \bbR \right\}.
\end{align*}

We want to show that 
\begin{align}
\label{eqn:equation1}
\min_{\beta_{d-1}} ||x_{(1)} - \beta_{d-1} S(x; d - 1, d) ||_\infty &= 1 / (2d - 1) \\
\label{eqn:equation2}
\min_{\beta_0, \beta_{d-1}} ||x_{(1)} - \beta_0 - \beta_{d-1} S(x; d - 1, d) ||_\infty &= 1 / (2d) \\
\label{eqn:equation3}
\min_{\beta_0, \beta_1, \hdots, \beta_{d-1}} || x_{(1)} - \beta_0 - \beta_1 S(x; 1, d) - \hdots - \beta_{d-1} S(x; d-1, d)||_\infty &\le 1 / 2^d
\end{align}

where the $L_\infty$ norm is taken over values of $x$. For \autoref{eqn:equation1}, we first show a lower bound:

\begin{align*}
\min_{\beta_{d-1}} ||x_{(1)} - \beta_{d-1} S(x; d - 1, d) ||_\infty &\ge \\
\min_{\beta_{d-1}} \max\{|x_{(1)} - \beta_{d-1} S((1, 0, \hdots, 0); d - 1, d)|, |x_{(1)} - \beta_{d-1} S((1, 1, \hdots, 1); d - 1, d)| \} &= \\
\min_{\beta_{d-1}} \max\{|1 - \beta_{d-1}(d - 1) / d|, |1 - \beta_{d-1}|\}.
\end{align*}


Evidently $\beta_{d-1} < 0$ cannot be optimal, thus $1 - \beta_{d-1}(d - 1)/ d > 1 - \beta_{d-1}$, and the optimal $\beta_{d-1}$ will be such that $1 - \beta_{d-1}(d - 1)/d > 0 > 1 - \beta_{d-1}$. To minimize the max of the two terms, we should equate the distance that $1 - \beta_{d-1}(d - 1)/d$ is above zero with the distance that $1 - \beta_{d-1}$ is below zero. I.e. 

\begin{align*}
1 - \beta_{d-1}(d - 1)/d = \beta_{d-1} - 1 \iff \beta_{d-1} = 2d / (2d - 1). 
\end{align*}

We show that this lower bound is tight by contradiction. Suppose that there is some $\beta'_{d-1}$ that achieves a criterion strictly less than the $1 / (2d - 1)$ achieved by $\beta_{d-1} = 2d / (2d - 1)$. This would imply that 

\begin{align*}
&-1 / (2d - 1) < 1 - \beta'_{d-1} \iff \beta'_{d-1} < 2d / (2d - 1) \textnormal{ and } \\
& 1 - \beta'_{d-1}\frac{d - 1}{d} < +1 / (2d - 1) 
\iff \frac{2d}{ 2d - 1} < \beta'_{d-1}.
\end{align*}

The proof of \autoref{eqn:equation2} is clearly quite similar:

\begin{align*}
\min_{\beta_0, \beta_{d-1}} ||x_{(1)} - \beta_0 - \beta_{d-1} S(x; d - 1, d) ||_\infty &\ge \\
\min_{\beta_0, \beta_{d-1}} \max\{|x_{(1)} - \beta_0 - \beta_{d-1} S((0, 0, \hdots, 0); d - 1, d)|, \\
                                  |x_{(1)} - \beta_0 - \beta_{d-1} S((1, 0, \hdots, 0); d - 1, d)|, \\
                                  |x_{(1)} - \beta_0 - \beta_{d-1} S((1, 1, \hdots, 1); d - 1, d)| \} &= \\
\min_{\beta_0, \beta_{d-1}} \max\{|- \beta_0|, |1 - \beta_0 - \beta_{d-1}(d - 1)/d|, |1 - \beta_0 - \beta_{d-1}|\} 
\end{align*}

Clearly, the optimal $\beta_0$ and $\beta_1$ are positive. Thus, we conjecture that optimality is characterized by the fitted value at $(1, 1, 0, \hdots, 0)$ being as positive in magnitude as the one at $(1, 0, \hdots, 0)$ is negative:

\begin{align*}
1 - \beta_0 - \beta_{d-1}(d - 1)/d = \beta_0 + \beta_{d-1} -1  \iff \\
2 = 2\beta_0 + \beta_{d-1}(2d - 1)/d  \iff \\
1 = \beta_0 + \beta_{d-1}(2d - 1)/(2d)  
\end{align*}

We further conjecture that the common value should also be equated to $\beta_0$, or 

\begin{align*}
1 - \beta_0 - \beta_{d-1} = + \beta_{d-1}/(2d) = \beta_0 \iff \\
1 = \beta_0 + \beta_0 (2d - 1) = 2d \beta_0 \iff \beta_0 = 1 / (2d).
\end{align*}

Thus, $(\beta_0, \beta_{d-1})= (1 / (2d), 1)$ , and $\beta_0$ is a lower bound on the criterion. As above, we show that the lower bound is tight by contradiction: assume that there are some $(\beta_0', \beta_{d-1}')$ that achieves a criterion $< 1 / (2d)$. Then (evaluating the error at 0) $1 - \beta_0' = 1 - 1 / (2d) + \varepsilon$ for some $\varepsilon > 0$. This implies that the criterion at $(1, 1, 0, \hdots, 0)$ satisfies:

\begin{align}
\label{eqn:1100}
-1 / (2d) < 1 - \beta_0' - \beta_{d-1}' = 1 - 1 / (2d) + \varepsilon- \beta_{d-1}' \iff \beta_{d-1}' < 1 + \varepsilon.
\end{align}

Whilst evaluating the criterion at $(1, 0, 0, \hdots, 0)$ implies that 

\begin{equation}
\begin{aligned}
\label{eqn:1000}
1 - \beta_0' - \beta_{d-1}' (d - 1) / d < +1 / (2d) \iff & \\
1 + \varepsilon - \beta_{d-1}' (d - 1) / d < 1 / d \iff & \\
(d - 1) / d + \varepsilon < \beta_{d-1}' (d - 1) / d \iff & \\
1 + \varepsilon \times d / (d - 1) < \beta_{d-1}'.
\end{aligned}
\end{equation}

\autoref{eqn:1100} and \autoref{eqn:1000} are clearly incompatible, thus the lower bound is tight.

\autoref{eqn:equation3} will be shown if we demonstrate a $\beta_0, \beta = (\beta_1, \hdots, \beta_{d-1})$ that achieves an error of $1 / 2^d$. Let $B(d)$ be the $d - 1\times d$ upper-diagonal matrix with $(r, c)$th element ${d - c \choose r - 1} / {d \choose r}$ if $r + c \le d + 1$, and zero otherwise.  We can write the condition \autoref{eqn:Sxrd} simultaneously for all $r$ as 

\begin{align*}
S(x; d) \triangleq \begin{pmatrix} S(x; 1, d) \\ S(x; 2, d) \\ S(x; 3, d) \\ \vdots \\ S(x; d-1, d) \end{pmatrix} = B(d) \begin{pmatrix}  
x_{(1)} \\ x_{(2)} \\ \vdots \\ x_{(d)}
\end{pmatrix}.
\end{align*}

\begin{align*}
x_{(1)} - \beta_0 - \beta^\top S(x; d) = x_{(1)} - \beta_0 - \beta^\top B(d) \begin{pmatrix} x_{(1)} \\ x_{(2)} \\ \vdots \\ x_{(d)}\end{pmatrix}.
\end{align*}

In this more concise notation, we need to compute a $\beta_0, \beta$ such that:

\begin{align*}
\max_{x \in [0, 1]^d} |x_{(1)} - \beta_0 - \beta^\top S(x; d)| = 1 / 2^d.
\end{align*}
The space of order statistics is difficult to compute with, so we exchange the $d$ inequalities that $1 \ge x_{(1)} \ge x_{(2)} \ge \hdots \ge x_{(d-1)} \ge x_{(d)} \ge 0$, for the single constraint that the variables sum to one via the \emph{V representation}: $\left\{y: y \in [0, 1]^d, y_i \ge y_{i+1} \right\} = \left\{V(d)\lambda : \lambda \in \Delta_d \right\}$, where $V(d)$ is the $d \times d + 1$ matrix 

\begin{align}
\label{eqn:Vd}
\begin{pmatrix} 
0 & 1 & 1 &\hdots & 1 &1 \\
0 & 0 & 1 &\hdots & 1 &1 \\
\vdots & \hdots & & \vdots \\
0 & 0 & 0 &\hdots & 1 &1 \\
0 & 0 & 0 &\hdots & 0 &1
\end{pmatrix}.
\end{align}

In this formulation, $x_{(1)} = 1 - \lambda_1 = \lambda_2 + \hdots + \lambda_{d+1}$, and the error as a function of $\beta_0, \beta = \begin{pmatrix}\beta_1 & \hdots & \beta_{d-1}\end{pmatrix}$ is $\max_{\lambda \in \Delta_{d}} | L(\beta) \lambda - \beta_0|$ where $L(\beta) = (\begin{pmatrix} 0 & 1 & \hdots 1 \end{pmatrix} - \beta^\top B(d) V(d))$.

This $\beta$ is given by $\beta^\star_0 = 1 / 2^{d} $, and 
\begin{align*}
 \beta^\star = -1 \times
\begin{pmatrix} 
(-1 / 2)^{d-1} {d \choose 1} \\ (-1 / 2)^{d-2} {d \choose 2}\\ \vdots \\ (-1/2)^1 {d \choose d}
\end{pmatrix},
\end{align*}
For such a $\beta^\star$, we have that

\begin{align*}
\beta^{\star \top} B(d) = \begin{pmatrix} 1 - 1/2^d & +1/2^d & -1/2^d & \hdots & (-1)^d 1/2^d \end{pmatrix}.
\end{align*}

Then $\beta^{\star \top} B(d)V(d)$ is the $1 \times d + 1$ vector starting with a zero, then a 1, and thereafter followed by alternating values of $1 - 1 / 2^d$ and $1$. Thus, 
\begin{align*}
\beta^\top S(x; d) = \begin{cases}
(\lambda_2 + \lambda_4 + \hdots + \lambda_{d+1}) + (1 - 1 / 2^d) (\lambda_3 + \lambda_5 + \hdots + \lambda_{d}) & \textnormal{ if } d + 1 \textnormal{ is even } \\
 (\lambda_2 + \lambda_4 + \hdots + \lambda_{d}) + (1 - 1 / 2^d) (\lambda_3 + \lambda_5 + \hdots + \lambda_{d+1})& \textnormal{ otherwise.}
\end{cases}
\end{align*}

\begin{equation}
\begin{aligned}
\label{eqn:lambda_weights}
& \max_{x \in [0, 1]^d} |x_{(1)} - \beta_0^\star - \beta^{\star \top} S(x; d)| \\
=& \max_{\lambda \in [0, 1]^{d+1}, 1_{d+1}^\top \lambda = 1} | (1 - \lambda_1) - \beta_0^\star - \beta^{\star \top} B(d) V(d) \lambda |  \\
=& \max_{\lambda \in [0, 1]^{d+1}, 1_{d+1}^\top \lambda = 1} | (\lambda_2 + \hdots + \lambda_{d+1}) - \beta_0 - (\lambda_2 + \lambda_4 + \hdots ) - (1 - 1 / 2^d) (\lambda_3 + \lambda_5 + \hdots)| \\
=& \max_{\lambda \in [0, 1]^{d+1}, 1_{d+1}^\top \lambda = 1} | (1 / 2^d) (\lambda_3 + \lambda_5 + \hdots)- 1 / 2^d |
\end{aligned}
\end{equation}

where in the last we have substituted the actual value of $\beta_0^\star$. For all $\lambda$ the first term is bounded between $0$ and $1 / 2^d$, thus setting $\lambda_3 = \lambda_5  = \hdots = 0$ maximizes the value and achieves a value of $1 / 2^d$. 

\end{proof}

\begin{lemma}[Fundamental perturbation analysis of $L_\infty$ optimization]
\label{thm:perturbation_analysis}
Given a vector $\gamma \in \bbR^{d}$:

\begin{align}
\label{eqn:obvious}
||\gamma||_\infty &= \max_{\lambda \in \Delta_d} |\gamma^\top \lambda| \\
\label{eqn:first}
(\gamma_{(1)} + \gamma_{(d)})/2 &= \argmin_{a \in \bbR} \max_{\lambda \in \Delta_d} |\gamma^\top \lambda - a| \\
\label{eqn:second}
(\gamma_{(1)} - \gamma_{(d)})/2 &= \min_{a \in \bbR} \max_{\lambda \in \Delta_d} |\gamma^\top \lambda - a|
\end{align}
\end{lemma}

\begin{proof}
%
\autoref{eqn:obvious} is completely evident, we state it just to introduce the notation and because it will be used subsequently. For $\lambda \in \Delta_d$, 
\begin{align*}
|\gamma^\top \lambda - a| = |(\gamma - 1_d \times a)^\top \lambda|
\end{align*}
and \autoref{eqn:first} follows from the optimality principle that the optimal $a$ above should set the largest value of $\gamma - 1_d \times a$ to be as far above zero as the smallest value is below zero. If this condition is not met, then the criterion can be further reduced by moving the more extreme value towards zero without changing the index of the argmin.

This is achieved by setting $a = (\gamma_{(1)} + \gamma_{(d)})/2$ so that the largest value of $\gamma - 1_d \times a$ is $(\gamma_{(1)} - \gamma_{(d)})/2$, and the smallest value is $(\gamma_{(d)} - \gamma_{(1)})/2$, with all other values in-between. This will be the achieved criterion, hence \autoref{eqn:second}. 

%
\end{proof}

\begin{theorem}
\label{thm:all_my_work}
Let $L(\beta) = (\begin{pmatrix} 0 & 1 & \hdots 1 \end{pmatrix} - \beta^\top B(d) V(d))$, then an optimal $\beta_0^\star, \beta^\star$ satisfies
\begin{align}
\label{eqn:amdw_assertion}
\beta^\star = \argmin_{\beta } \max_i \ L(\beta)_i. 
\end{align}
And $L(\beta^\star) \ge 0, \beta_0^\star = \max_i \ L(\beta^\star)_i$.
\end{theorem}

\begin{proof}
Suppose that at a candidate $\beta$, $m \triangleq \max_i\ (\beta^\top B(d) V(d))_i > 1$. By \autoref{eqn:second}, the attained criterion will then be

\begin{align*}
(\max_i \ L(\beta)_i - (1 - m)) / 2.
\end{align*}

Note that the first column of $B(d) V(d)$ is entirely zero, so the first element of $\beta^\top B(d) V(d)$ is always zero. Thus, $\max_i \ L(\beta)_i = \max(0, 1 - \min_i \ (\beta^\top B(d) V(d))_i)$. And since for $\beta / m$, $\max_i \ L(\beta / m)_i \ge \max_i \ L(\beta)_i$, we have that 

\begin{align*}
& \max_i \ L(\beta / m)_i \le (\max_i \ L(\beta)_i - (1 - m)) \iff (1 - m) \le \max_i \ L(\beta)_i - \max_i \ L(\beta / m)_i.
\end{align*}

We need to consider three separate cases:

\begin{enumerate}
\item $0 = \max_i \ L(\beta / m)_i \implies 0 = \max_i \ L(\beta)_i$ in which case the inequality holds strictly. 
\item $L(\beta)_i = 0$ but $ \max_i \ L(\beta / m)_i > 0$ then $\max_i \ L(\beta / m)_i = 1 - \frac{1}{m} \min_i \ (\beta^\top B(d) V(d))_i$. This holds only if $ \min_i \ (\beta^\top B(d) V(d))_i > 1$, so 
\begin{align*}
1 - m \le -1 \times \left( 1 - \frac{1}{m} \min_i \ (\beta^\top B(d) V(d))_i \right) \iff 2 < m + \min_i \ (\beta^\top B(d) V(d))_i / m 
\end{align*}
which follows from the AM-GM inequality: $a > 1, b > 1 \implies (a + b / a) / 2 \ge \sqrt{b} > 1$.
\item If both terms are nonzero, then 
\begin{align*}
1 - m \le \frac{1 - m}{m} \times \min_i \ (\beta^\top B(d)V(d))_i \iff m \ge \min_i (\beta^\top B(d)V(d))_i.
\end{align*}
\end{enumerate}

Thus, at an optimal $\beta^\star$, we have that $||\beta^{\star \top} B(d)V(d)||_\infty \le 1$. And the assertion \autoref{eqn:amdw_assertion} follows from \autoref{eqn:first}. Since, $L(\beta^\star) \ge 0$, with the first entry being identically zero, then we have (from \autoref{eqn:second}) that the minimum is always zero and $\beta_0^\star = \max_i \ L(\beta)_i$. 
\end{proof}

\begin{lemma}
\label{lemma:increasing}
An optimal $f$ is increasing.
\end{lemma}

\begin{proof}
Two points $x_1, x_2$ will have $x_1 \ge x_2$ if and only if their V representations are similarly ordered, thus $f$ will be increasing if and only if $\beta^\top B(d) V(d) \ge 0$.

Suppose otherwise, that for some $i$ $(\beta^\top B(d) V(d))_i < 0$. Then the criterion will be $\ge (1 - \beta^\top B(d) V(d))_i) / 2 > 1 / 2$ by \autoref{eqn:first}. However, by setting $\beta = 0$, a criterion of $1 / 2$ can always be achieved, thus $\beta$ cannot be optimal.
\end{proof}

\subsection{Proof of \autoref{lemma:nonnegligible_measure}}
\label{sec:proof_nonnegligible_measure}

\begin{proof}

\autoref{lemma:increasing} shows that an optimal estimator is weakly increasing. Thus, for $\delta < \beta_0^\star$, $p \triangleq (\delta, 0, \hdots, 0) \implies f(p) \ge \beta_0 = f(0) = \beta_0^\star > \delta = \max(p)$ and means an error of at least $\beta_0^\star - \delta$. So at $p$, the error will be at least $\epsilon$ iff $\beta_0^\star - \epsilon \ge \delta$. This is true for all $p \in [0, \delta]^d$, thus $[0, \beta^\star_0 - \epsilon]^d\subseteq W(\epsilon; R) \implies \vol(W(\epsilon; R)) \ge (\beta^\star_0 - \epsilon)^d$. \autoref{thm:all_my_work} shows that $\beta_0^\star = \err(R) / 2$, and the assertion is proven. 

\end{proof}

This simplistic volume bound could straightforwardly be improved by including not just the intercept in the computation of the bound. And, as one might intuit, a similar analysis holds at all vertices. 
\subsection{Proof of \autoref{thm:main_lower_bound}}
\begin{proof}

Assume first that the functional decomposition is in normalized form. To the $k$th $r$-tuple in $\{1, 2, \hdots, d\}$, $C(k, r, d)$, associate a ``sliver'' of the unit cube $[0, 1]^d$. Let $F(c) = \{x \in [0, 1]^d : x_{c_1} \ge x_{c_2} \ge \hdots \ge x_{c_r} \ge x_{j}\textnormal{ for all } j \not\in c\}$ be a mapping $\{1, 2, \hdots, d\}^r \rightarrow [0, 1]^d$ that extracts the subset of $[0, 1]^d$ where the coordinates in $c$ are in decreasing order, and all the remaining coordinates are less than $x_{c_r}$. Let

\begin{align*}
\Xi_1 &= F((1, \hdots, r)) 
\textnormal{ and for } k > 1\\
\Xi_k &= F(C(k, r, d)) \backslash \cup_{j = 1}^{k - 1} \Xi_j.
\end{align*}


Let $H(x; c) = \sum_{j = 1}^{d - r + 1} {d - j \choose r - 1} x_{c_{j}}$. For all $x \in \Xi_k$, by construction $S(x; r, d) = H(x; C(k, r, d))$, thus using the indicator function to partition the unit cube along each sliver, we see that:

\begin{equation}
\begin{aligned}
\label{eqn:equiv_slivers}
S(x; r, d) &= \underbrace{\sum_{k=1}^{{d \choose r}} \chi_{\Xi_k}(x)}_{ = 1} \times \frac{1}{{d \choose r}} \sum_{j = 1}^{d - r + 1} {d - j \choose r - 1} x_{(j)}\\
           &= \frac{1}{{d \choose r}} \sum_{k=1}^{{d \choose r}} \chi_{\Xi_k}(x) \times H(x; C(k, r, d)).
\end{aligned}
\end{equation}

Let $U^n(x; r, d) = \sum_{j=1}^{n} \chi_{\Xi_j}(x) \times H(x; C(j, r, d))$, and let 

\begin{align*}
G^n(x; C(k, r, d)) = \begin{cases} \chi_{\Xi_k}(x) \times H(x; C(k, r, d)) & \textnormal{ if } k \ge n \\
                                                                         0 & \textnormal{ otherwise }. 
\end{cases}
\end{align*}

We show by induction that $x \mapsto G^n(x; C(k, r, d)), k = 1, \hdots, {d \choose r}$ is the unique functional decomposition of $U^n$. For $n = 1$

\begin{align*}
U^1(x; r, d) &= \chi_{\Xi_1}(x) \times H(x; C(1, r, d)) + 0 + \hdots + 0 \\
             &= \sum_{k=1}^{{d \choose r}} G^1(x; C(k, r, d))
\end{align*}

thus the base case of the inductive hypothesis holds. Suppose that 
\begin{align*}
U^{n-1}(x; r, d) = \sum_{k=1}^{{d \choose r}} G^{n-1}(x; C(k, r, d)) 
\end{align*}

is the unique functional decomposition of $U^{n-1}$. Then 
\begin{align}
\label{eqn:inductive_line1}
U^n(x; r, d) &= U^{n-1}(x; r, d) + \chi_{\Xi_n}(x) \times H(x; C(n, r, d)) \\ 
\label{eqn:inductive_line2}
             &= \sum_{k=1}^{{d \choose r}} G^{n-1}(x; C(k, r, d)) + \chi_{\Xi_n}(x) \times H(x; C(n, r, d)) \\ 
\label{eqn:inductive_line3}
             &= \sum_{k=1}^{n - 1} G^{n-1}(x; C(k, r, d)) + \chi_{\Xi_n}(x) \times H(x; C(n, r, d)) \\ 
\label{eqn:inductive_line4}
             &= \sum_{k=1}^{n - 1} G^{n}(x; C(k, r, d)) + \chi_{\Xi_n}(x) \times H(x; C(n, r, d)) \\ 
\label{eqn:inductive_line5}
             &= \sum_{k=1}^{n - 1} G^{n}(x; C(k, r, d)) + G^{n}(x; C(n, r, d)) \\ 
\label{eqn:inductive_line6}
             &= \sum_{k=1}^{{d \choose r}} G^{n}(x; C(k, r, d)) 
\end{align}

where \autoref{eqn:inductive_line1} is by the definition of $U^n$, \autoref{eqn:inductive_line2} is by the inductive hypothesis, \autoref{eqn:inductive_line4} follows because $G^n(x; C(k, r, d)) = G^{n-1}(x; C(k, r, d))$ if $k \le n - 1$, and \autoref{eqn:inductive_line5} is by the definition of $G^n$. \autoref{eqn:inductive_line3} and \autoref{eqn:inductive_line6} follow simply by recognizing which terms of the sum are zero by definition. Thus, $G^n$ is a functional decomposition of $U^n$. 

To see that it is unique, suppose that $f_1, \hdots, f_{{d \choose r}}$ is another functional decomposition with $f_j(x) \neq G^n(x; C(j, r, d))$. If $j = n$, then $f_n(x) \neq H(x; C(n, r, d))$ which contradicts the definition of $H$. If $j \neq n$, then this contradicts the inductive hypothesis, because then $G^{n-1}$ could not be a functional decomposition of $U^{n-1}$. 

Since $x \mapsto G^n(x; C(k, r, d)), k = 1, \hdots, {d \choose r}$ is the unique functional decomposition of $U^n$, for all $n$, in particular, $x \mapsto G^{d \choose r}(x; C(k, r, d)), k = 1, \hdots, {d \choose r}$ is the unique functional decomposition of $U^{{d \choose r}} = {d \choose r} \times S$. And since $G^{{d \choose r}} \neq 0$ for all $k$, no function in a normalized functional decomposition of $S$ is zero.

The process or putting a general functional decomposition into normalized form weakly reduces the number of nonzero functions, thus \autoref{thm:main_lower_bound} straightforwardly holds for general functional decompositions. 

Since evaluating an $\{0, 1, 2, \hdots, d-1\}$-estimator thus entails compuing ${d \choose 1} + {d \choose 2} + \hdots + {d \choose d - 1} $ terms, and is therefore of $O(2^d)$ complexity. 

\end{proof}

\subsection{Proof of \autoref{thm:d1_width}}

\begin{proof}

In order to formally analyze this problem, we identify neurons with integer tuples $\in \{1,2, \hdots, d\}^{\zeta(d, j)}$ representing the indices over which subpool maxes are taken. Within a layer, the maximum of two neurons computes the value taken by a neuron in the subsequent layer, with that neuron indicated by the union of the two upstream neuron's integer tuples. 

Let $T(d, j)$ denote the neurons at the $j$th layer of a network $j = 1, 2, \hdots, \textsc{depth}(d)$. Then $w(d, j)$ is the number of elements in $T(d, j)$ and implementing a $\{d-1\}$-estimator requires that the final set of tuples be all $d$ tuples of length $d - 1$ with each element dropped in turn:

\begin{align*}
T(d, \textsc{depth}(d)) = \{&(1, 2, \hdots, d - 1), \\
                            &(1, 2, \hdots, d - 2, d), \\
                            &\hdots, \\
                            &(1, 3, \hdots, d- 1, d), \\
                            &(2, 3, \hdots, d - 1, d)\}.
\end{align*}

Because each element of $T(d, \textsc{depth}(d))$ differs only by a single term, for all $1 \le i_1 < i_2 \le d$, 

\begin{equation}
\begin{aligned}
\label{eqn:reuse_condition}
\{(t_{i_1}, \hdots, t_{i_2}) : t \in T(d, \textsc{depth}(d)) \} &= \{(i_1, i_1 + 1, \hdots, i_2 - 1), \\
                                                   & \ \ \ \ \ \ \ (i_1, i_1 + 1, \hdots, i_2 - 2, i_2), \\
                                                   & \ \ \ \ \ \ \ \ \ \vdots \\   
                                                   & \ \ \ \ \ \ \ (i_1 + 1, i_1 + 2, \hdots, i_2 - 1, i_2)\}
\end{aligned}
\end{equation}

which is of size $i_2 - i_1 + 1$. Let $\psi$ be the function that splits increasing pairs of integers at a midpoint. For $A \subseteq \{a: a_2 > a_1\} \subseteq \{1, 2, \hdots, d\}^{2}$, and $o(m) = 1$ if $m$ is odd, and zero otherwise:

\begin{align*}
\psi(A, s) = \left\{(a_{1}, a_{1} + s) : a \in A  \right\} \cup \left\{(a_{1} - o(a_2 - a_1) + s, a_2) : a \in A \right\}.
\end{align*}

Let $\Psi(d, \textsc{depth}(d)) = \{(1, d - 1)\}$ and for $j = 1, 2, \hdots, \textsc{depth}(d) - 1$, let
\begin{align*}
\Psi(d, \textsc{depth}(d) - j) = \psi(\Psi(d, \textsc{depth}(d) - j + 1), \zeta(d, j))
\end{align*}

be defined recursively as the set of ``split points'' that arise from progressively splitting $\Psi(d, \textsc{depth}(d))$ at $\zeta(d, 1), \zeta(d, 2),$ etc. 

Finally, for $j = 1, 2, \hdots, \textsc{depth}(d)$ let $T(d, j) = \{(t_{i_1}, \hdots, t_{i_2}) : t \in T(d, \textsc{depth}(d)), (i_1, i_2) \in \Psi(d, j) \}$ be the set of all final tuples split at the points given by $\Psi(d, j)$. 
The notation is heavy, but idea is simple: a set of subpools is generated at the $j$th layer of the network by splitting each element of the terminal decomposition according to the split points computed by $\Psi(d, \textsc{depth}(d))$. The neurons at that layer have thus computed the subpool maxes of size $\zeta(d, j)$ by the $j$th layer. The size of the subpools, $\zeta(d, j)$ double with each $j$, and the split occurs in the middle of the existing tuple indices. Refining the indices by iterating backwards insures that each element of $T(d, j)$ is formed from the pairwise union of elements in $T(d, j-1)$, and thus can be computed by a forward pass through the network. 

The size of the subpool and the number of terms in the decomposition are connected by the equation that the size of $T(d, j)$ is $2^{\textsc{depth}(d) - j} \times (1 + \zeta(d, \textsc{depth}(d) - j))$. This is the product of (1) how many tuples each original tuple is split into, $2^{\textsc{depth}(d) - j}$, and (2) the number of unique tuples of that size, $1 + \zeta(d, j)$. The first term is because each backwards step taken from $\textsc{depth}(d)$ doubles the number of tuples. The second term follows by \autoref{eqn:reuse_condition}.

The expression for $w(d, 1)$ follows from a dichotomization of $T(d, 1)$ into two types: there must evidently be $d - 1$ terms of the form $(k, k + 1)$ for $k = 1, 2, \hdots,  d- 1$ for all $d$. There will be a term of the form $(k, k + 2)$ for each element of $\Psi(d, 1)$, which is straightforwardly seen to be of size $2^{\lfloor \log_2(d - 2)\rfloor}$.

\end{proof}


\section{$L_2$ problem}
\label{sec:main_theorem_width_L2}

For brevity in what follows, let $K(d) = B(d) V(d) \in \bbR^{d -1 \times d + 1}$, then the difference between the fitted and actual values, as a function of $\lambda$ is:

\begin{align}
\label{eqn:residual_simplex}
\lambda \mapsto (1_{d+1} - e_{d+1,1})^\top \lambda - \beta_0 - \beta^\top K(d) \lambda = - \beta_0 + ((1_{d+1} - e_{d+1,1}) - K(d)^\top\beta)^\top \lambda .
\end{align}

From \autoref{eqn:residual_simplex} the squared $L_2$ error is 

\begin{align}
\label{eqn:error_concrete}
\int_{\Delta_d}(\beta_0 - ((1_d - e_{d1}) - K(d)^\top \beta)^\top \lambda )^2 \mathrm{d} \lambda.
\end{align}

To lighten the notation, we wrap this novel optimization problem into \autoref{lemma:quadratic_problem}.

\begin{lemma}
\label{lemma:quadratic_problem}
Let $\alpha_0 \in \bbR$, $\alpha \in \bbR^d$, $A \in \bbR^{d + 1}$, and $\Xi \in \bbR^{d + 1 \times d}$. Let $v(d) =  \int_{\Delta_d} \dif \lambda$, then

\begin{align}
\label{eqn:quadratic_criterion_abstract}
\min_{\alpha_0, \alpha}\ \int_{\Delta_d} (\alpha_0 - (A - \Xi \alpha)^\top \lambda)^2 \dif \lambda = v(d) A^\top \left(I - \Sigma(d) \Xi \left(\Xi^\top \Sigma(d)\Xi \right)^\dagger \Xi^\top \Sigma(d) \right)A.
\end{align}
\end{lemma}

\begin{proof}

Expanding the criterion above:
\begin{align*}
\alpha_0^2 v(d) - 2\alpha_0 (A - \Xi\alpha)^\top \left( \int_{\Delta_d} \lambda \dif \lambda \right)  + (A - \Xi\alpha)^\top \left(\int_{\Delta_d} \lambda\lambda^\top \dif \lambda \right) (A - \Xi \alpha).
\end{align*}

The first order criterion for optimality of $\alpha_0$ evidently requires that 

\begin{align*}
\alpha_0 = (A - \Xi\alpha)^\top \left( \int_{\Delta_d} \lambda \dif \lambda \right) / v(d) ,
\end{align*}

thus, the criterion equals 

\begin{align*}
v(d) \times (A - \Xi \alpha)^\top \left( \int_{\Delta_d} \lambda \lambda^\top \dif \lambda / v(d) - \int_{\Delta_d} \lambda \dif \lambda / v(d) \int_{\Delta_d} \lambda^\top \dif \lambda / v(d) \right) (A - \Xi \alpha). 
\end{align*}
Write the inner term -- the covariance matrix of a  Dirichlet$(1, 1, \hdots, 1)$ distribution -- as $\Sigma(d)$, then this weighted least squares problem is solved by 
\begin{align*}
\alpha^\star = \left(\Xi^\top \Sigma(d)\Xi \right)^\dagger \Xi^\top \Sigma(d)A.
\end{align*}
Plugging this equation into the criterion gives \autoref{eqn:quadratic_criterion_abstract}.
\end{proof}

Phrasing \autoref{eqn:error_concrete} in terms of \autoref{eqn:quadratic_criterion_abstract}, we have that the squared $L_2$ error of the optimal coefficients is:
\begin{align}
v(d) (1_d - e_{d1})^\top \left(I - \Sigma(d) K(d)^\top \left(K(d) \Sigma(d)K(d)^\top \right)^\dagger K(d) \Sigma(d) \right)(1_d - e_{d1}).
\end{align}

The hat matrix $\Sigma(d) K(d)^\top \left(K(d) \Sigma(d)K(d)^\top \right)^\dagger K(d) \Sigma(d)$ has rank $d - 1$, thus $(1_d - e_{d1})$ cannot possibly lie in the nullspace of the projection operator, and we have a strictly positive error. We skip deriving the exact expressions as a function of $d$ and note that the same analysis could be straightforwardly conducted constraining different coefficients to equal zero. 

%


\section{Optimal approximation error}
\label{sec:optimal_approximation_error}
\autoref{fig:all_subset_approx} gives an upper bound on the error of optimal estimators as a function of $R$ for $d = 2, 3, 4$. These values are computed numerically using the construction \autoref{eqn:criterion}. 

\begin{figure}
\begin{center}
\begin{tabular}{|l|l|l|}
\hline 
$d$ & $R$ & error \\
\hline 
2 & $\{1 \}$        & 1/3 \\
2 & $\{0, 1 \}$       & 1/4 \\
\hline 
3 & $\{1 \}$         & 1/2 \\
3 & $\{2 \}$        & 1/5 \\
3 & $\{0, 1\}$       & 1/3 \\
3 & $\{0, 2\}$       & 1/6 \\
3 & $\{1, 2\}$       & 1/7 \\
3 & $\{0, 1, 2\}$    & 1/8 \\
\hline 
4 & $\{1\}$        & 3/5 \\
4 & $\{2\}$         & 1/3 \\
4 & $\{3\}$         & 1/7 \\
4 & $\{0, 1\}$       & 3/8 \\
4 & $\{0, 2\}$       & 1/4 \\
4 & $\{0, 3\}$       & 1/8 \\
4 & $\{1, 2\}$       & 1/5 \\
4 & $\{1, 3\}$       & 1/9 \\
4 & $\{2, 3\}$       & 1/13 \\ 
4 & $\{0, 1, 2\}$    & 1/6  \\ 
4 & $\{0, 1, 3\}$    & 1/10 \\
4 & $\{0, 2, 3\}$    & 1/14 \\
4 & $\{1, 2, 3\}$    & 1/15 \\
4 & $\{0, 1, 2, 3\}$ & 1/16 \\
\hline 
\end{tabular}
\end{center}
\caption{All subset approximation errors for $d = 2, 3, 4$ (the error for $R = \{0\}$, which is 1 / 2 for all $d$ is omitted)}
\label{fig:all_subset_approx}
\end{figure}

\section{Experimental Details}

\label{sec:appendix:experiments}

\begin{figure}[t!]
    \centering
   \includegraphics[width=0.95\textwidth]{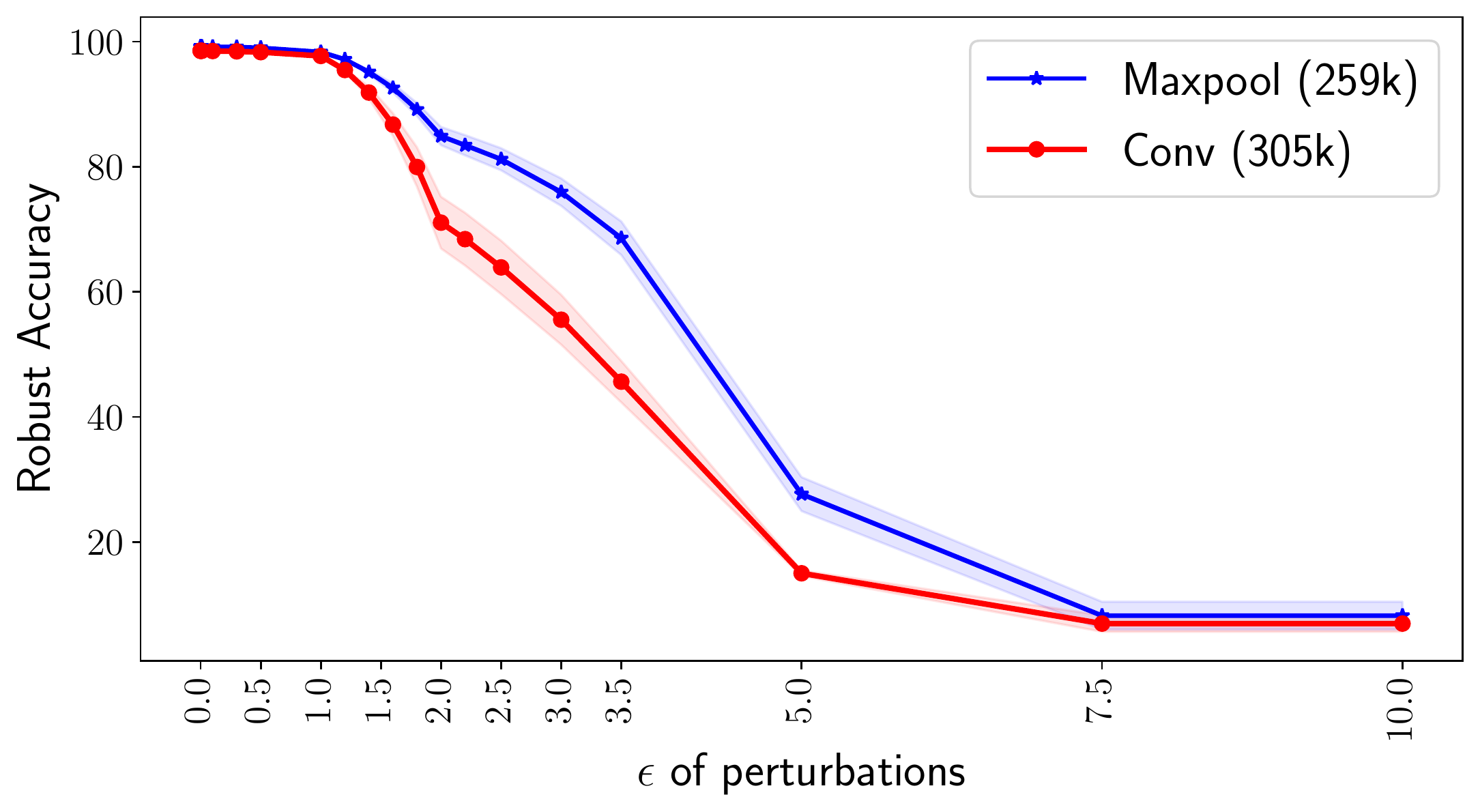}
    \caption{Effect of perturbations with \( \ell_\infty \leq \epsilon \) on the accuracy of the \lenet \ model variants on the \mnist\ dataset. }
    \label{fig:mnist-experiment}
\end{figure}

In this section, we present in greater detail our experimental framework. The goal is to show the greater adversarial robustness of models \textit{with} max pool layers than those \textit{without}. Specifically, Convolutional Neural Nets (CNNs) and Residual Networks (ResNets) are trained on the \mnist \ and \cifar data sets. Then, an adversarial attack is performed for various levels of perturbations, denoted by \( \epsilon \), and the robust accuracy (mean and standard deviation) of the various models is reported over three different random seeds. We present results for the Fast Gradient Sign Method \cite{fgsm} attack.

In each experiment, we create a model incorporating max pool layer(s). Then, the network is modified by replacing each max pool layer with a trainable variant, ensuring that the output of the original layer and the modified one have the same shape.

The legend of each figure presents the name of the model variant as well as the number of trainable parameters in parentheses. The width of the lines is proportional to the number of parameters in the model. The max pool model variant is always depicted in blue. The experiments are developed in PyTorch \cite{Paszke2017} and PyTorch Lightning with the help of the foolbox \cite{Rauber2017} library for the adversarial attack.  

\subsection{Experimental configurations}

\paragraph{\lenet \ experiment on \mnist}
We train two convolutional neural networks (CNNs) on the digit classification dataset \mnist. The results are shown in \autoref{fig:mnist-experiment}.

The first model, in blue, has 259\,106 trainable parameters and consists of two convolutional layers, with 32 and 64 channels. Their kernel size is equal to five. Both layers are succeeded by a two-dimensional max pool with kernel size, stride and padding equal to three, two and one, respectively. The network is completed with two fully-connected layers of 1024 and 200 neurons, leading to the output of ten logits. The second model, in red, has 305\,282 trainable parameters and has the same structure as the previous model. However, the max pool layers are replaced by a convolutional layer with the same number of input channels as the output of the preceding convolutional layer, hence the increase in parameters. Both models are trained for ten epochs of Stochastic Gradient Descent with learning rate and momentum equal to 0.01 and 0.9, respectively.

\paragraph{\lenet \ experiment on \cifar}
A similar experiment is performed on the more challenging \cifar\ dataset. 
The results are shown in \autoref{fig:cifar-experiment}.

The max pool model, in blue, has 232\,162 parameters and consists of three convolutional layers of 32, 64, 128 channels, respectively. Again, each of these layers is succeeded by max pool module identical to the \mnist\ experiment. The convolutional variant, depicted in green, has 253\,890 parameters has a similar modification as before, i.e. the max pool layer is replaced by a convolutional one with kernel size, stride and padding identical to the corresponding max pool layer. Finally, the strided variant, in gray, has the same number of parameters as the original max pool model. In this case, we replace the block of convolution and max pool with a strided convolution of stride equal to two. Hence, this modification does not incur an increase in number of parameters, while maintaining the same output shape at all intermediate steps. The models are trained for 100 epochs of SGD with learning rate and momentum equal to 0.01 and 0.9 respectively. The learning rate is decayed by a parameter \( \gamma=10 \) in epochs 50, 70 and 90. The batch size is 128.

\begin{figure}[t]
    \centering
   \includegraphics[width=0.95\textwidth]{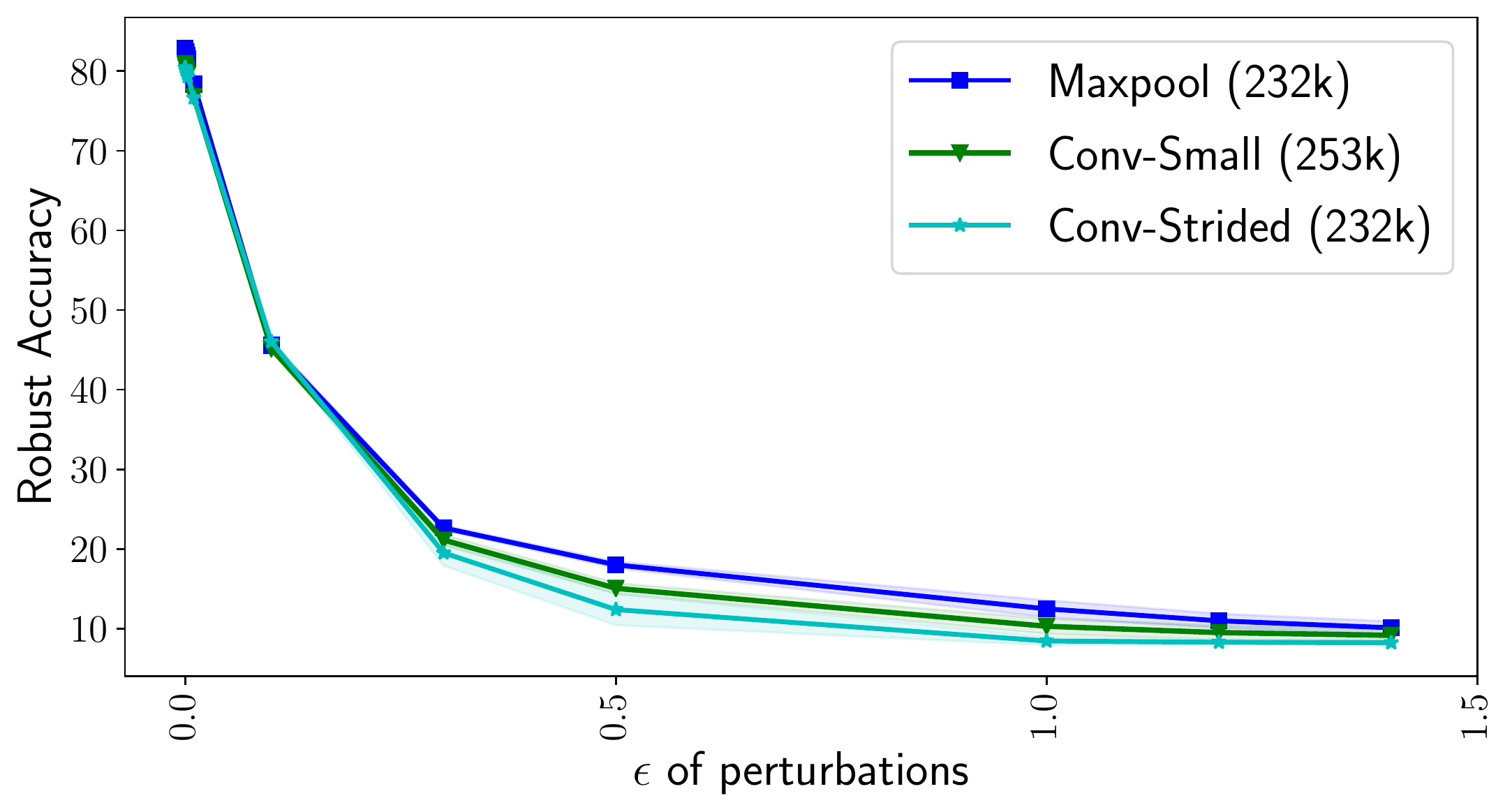}
    \caption{Effect of perturbations with \( \ell_\infty \leq \epsilon \) on the accuracy of the \lenet \ model variants on the \cifar\ dataset. }
    \label{fig:cifar-experiment}
\end{figure}

\paragraph{\resnet \ experiment on \cifar}

We use the \resnet\ variant proposed in \cite{Page2018_8}. 
This model consists of a preparatory whitening layer
, three residual blocks and a classifier layer. First, the preparatory layer has two convolutions and Ghost Batch Normalization \cite{Hoffer_Hubara_Soudry_2017}. The three residual blocks have identical structure, with the exception of the number of channels in the convolutional layer; the channels are doubled with each layer, from 64 to 128 to 256. Each of these layers consists of two blocks: the first one has a convolutional layer, a max pool layer with kernel size and stride equal to two and Ghost Batch Norm, while the second block employs a residual connection with similar structure as the previous block modulo the max pool. Finally, the classifier layer is comprised of a max pooling layer of kernel size and stride equal to four and a fully connected layer resulting in ten outputs.

Overall, the Maxpool model, in blue, has four maxpooling layers and 4\,666\,265 parameters. The large convolutional model, depicted in red, has 10\,238\,233 parameters and replaces the Maxpool layers by convolutional ones with the same kernel size, stride and padding. The small convolution variant, depicted in green, has 5\,273\,881 parameters and the kernel size is set to one for all maxpooling replacements. Finally, the strided model, depicted in gray, has 4\,928\,921 parameters and. for each residual layer, the pair of convolution and maxpooling is replaced by a strided convolution, while the maxpooling layer preceding the fully connected one is replaced with a convolutional layer of kernel size equal to one. The stride remains equal to four, as in the corresponding maxpooling layer.

The models are trained for 50 epochs using float 16 precision. The learning rate follows a piecewise linear schedule; starting at zero the learning rate linearly increases to 0.4 until the fifth epoch and then linearly decays to zero until the final epoch.


\subsection{Analysis}

Our objective lies in showing the superior adversarial robustness of models incorporating max pooling. In each experiment, we use a max pooling model, drawn from widely used neural networks such as \lenet \ and \resnet. We modify the original model to produce comparisons. Specifically, the modifications simply replace the maxpooling layer with a convolutional layer, or the pair of convolutional layer and maxpooling (which traditionally come succession) with a strided convolution. In both cases, the new layer produces outputs of the same shape as the original layer, lending itself to an one-to-one comparison in terms of performance on (robust) accuracy. It is important to note that the first modification results in an increase in the number of trainable parameters. Subsequently, we perform an adversarial attack, FGSM in our case, to illuminate the adversarial robustness properties of each model.
A common theme of all experiments is that the exclusion of the maxpooling layer results in a tradeoff between (robust) accuracy and model complexity.

First, in the \mnist\ experiment both variants reach similar levels of performance on the clean accuracy (\( \epsilon=0 \)); the max pool variant achieves $99.19 \pm 0.06$ while the convolutional model $98.54 \pm 0.14$. This is not surprising given the low difficulty of the dataset. Nevertheless, the model \textit{with} maxpooling is characterized by strictly higher adversarial robustness, since the difference in performance heightens for larger \( \epsilon \). In the \cifar\ experiment, the observations are similar in nature; replacing maxpooling with a trainable layer renders the model more susceptible to adversarial perturbations. However, the modified models do not exhibit the same level of clean accuracy, despite the increase in model complexity. Specifically, the mean clean accuracies (over 3 runs with different random seeds) of the max pool, small convolutional and strided convolution models are \( 82.89 \pm 0.08\% \), \(  80.94 \pm 0.26\% \) and \( 80.45\pm 0.66\%  \), respectively. It is important to note that the \lenet\ architecture does not achieve state-of-the-art results on any of the model variants presented. However, it serves as a direct comparison with the previous experiment. Finally, the \resnet\ experiment perhaps illuminates the tradeoff more clearly. \autoref{fig:resnet-experiment-main-text} (see main text) presents a dichotomy due to the exclusion of the max pool layer; the practitioner should choose between model complexity (measured in number of trainable parameters and, by extension, training and inference times) and (robust) accuracy. The large convolution model achieves a clean accuracy of \( 93.37 \pm 0.14\% \) compared to \( 94.49 \pm 0.20\%  \) of the original model and is able to match its robust accuracy for different \( \epsilon \), while using more than double the parameters. The other two variants, however, have lowest clean accuracies (\( 89.22 \pm 0.09\%  \) for the strided model and \( 87.46 \pm 0.99\%  \) for the small convolutional) and present a faster deterioration in adversarial robustness.

\subsection{Detailed Results}

For completenes, we present the experimental results in tabular form. The experiments were repeated three times (with different random seeds) and the mean \( \pm \) standard deviation is reported.
    
\begin{table}[htpb]
    \centering
    \caption{Detailed results on \mnist.}
    \label{tab:mnist}
    
    \begin{tabular}{c|cc}
        \toprule
        \( \epsilon \) & Maxpool & Conv \\
        \midrule
        0.000  &  $99.19 \pm 0.06$ &  $98.54 \pm 0.14$ \\
        0.001  &  $99.19 \pm 0.06$ &  $98.54 \pm 0.14$ \\
        0.002  &  $99.19 \pm 0.06$ &  $98.54 \pm 0.14$ \\
        0.003  &  $99.19 \pm 0.06$ &  $98.54 \pm 0.14$ \\
        0.010  &  $99.19 \pm 0.06$ &  $98.54 \pm 0.14$ \\
        0.100  &  $99.18 \pm 0.05$ &  $98.50 \pm 0.16$ \\
        0.300  &  $99.14 \pm 0.04$ &  $98.45 \pm 0.16$ \\
        0.500  &  $99.00 \pm 0.07$ &  $98.33 \pm 0.11$ \\
        1.000  &  $98.34 \pm 0.09$ &  $97.72 \pm 0.07$ \\
        1.200  &  $97.16 \pm 0.08$ &  $95.47 \pm 0.15$ \\
        1.400  &  $95.14 \pm 0.38$ &  $91.86 \pm 0.97$ \\
        1.600  &  $92.59 \pm 0.72$ &  $86.74 \pm 1.81$ \\
        1.800  &  $89.16 \pm 1.05$ &  $79.97 \pm 3.15$ \\
        2.000  &  $84.90 \pm 1.46$ &  $71.07 \pm 4.14$ \\
        2.200  &  $83.44 \pm 1.62$ &  $68.43 \pm 4.22$ \\
        2.500  &  $81.20 \pm 1.77$ &  $63.91 \pm 4.27$ \\
        3.000  &  $75.95 \pm 2.20$ &  $55.55 \pm 3.99$ \\
        3.500  &  $68.58 \pm 2.69$ &  $45.65 \pm 3.32$ \\
        5.000  &  $27.67 \pm 2.71$ &  $15.02 \pm 0.46$ \\
        7.500  &   $8.24 \pm 2.22$ &   $6.98 \pm 1.31$ \\
        10.000 &   $8.24 \pm 2.22$ &   $6.98 \pm 1.31$ \\
        \bottomrule
    \end{tabular}
\end{table}

\begin{table}[htpb]
    \centering
    \caption{Detailed results on \cifar\ with \lenet.}
    \label{tab:cifar-lenet}
    \begin{tabular}{c|ccc}
        \toprule
        \( \epsilon \) &           Maxpool &        Conv-small &           Strided \\
        \midrule
        0.000  &  $82.89 \pm 0.09$ &  $80.94 \pm 0.27$ &  $80.45 \pm 0.67$ \\
        0.001  &  $82.42 \pm 0.11$ &  $80.51 \pm 0.19$ &  $80.08 \pm 0.60$ \\
        0.002  &  $82.00 \pm 0.08$ &  $80.09 \pm 0.20$ &  $79.70 \pm 0.58$ \\
        0.003  &  $81.56 \pm 0.13$ &  $79.68 \pm 0.21$ &  $79.31 \pm 0.66$ \\
        0.010  &  $78.37 \pm 0.03$ &  $76.80 \pm 0.17$ &  $76.64 \pm 0.30$ \\
        0.100  &  $45.56 \pm 0.54$ &  $45.16 \pm 0.41$ &  $46.02 \pm 0.30$ \\
        0.300  &  $22.64 \pm 0.32$ &  $21.12 \pm 0.68$ &  $19.51 \pm 1.61$ \\
        0.500  &  $18.00 \pm 0.42$ &  $15.05 \pm 0.65$ &  $12.39 \pm 1.96$ \\
        1.000  &  $12.48 \pm 1.08$ &  $10.29 \pm 0.96$ &   $8.44 \pm 0.52$ \\
        1.200  &  $10.97 \pm 0.90$ &   $9.50 \pm 0.90$ &   $8.28 \pm 0.21$ \\
        1.400  &  $10.10 \pm 0.81$ &   $9.15 \pm 0.75$ &   $8.22 \pm 0.46$ \\
        \bottomrule
        \end{tabular}
\end{table}

\begin{table}[htpb]
    \centering
    \caption{Detailed results on \cifar\ with \resnet.}
    \label{tab:cifar-resnet}
    \begin{tabular}{c|cccc}
        \toprule
        \( \epsilon \) &           Maxpool &        Conv-Large &        Conv-Small &           Strided \\
        \midrule
        0.000 &  $94.49 \pm 0.20$ &  $93.37 \pm 0.14$ &  $87.46 \pm 0.99$ &  $89.22 \pm 0.09$ \\
        0.001 &  $94.03 \pm 0.26$ &  $92.95 \pm 0.16$ &  $87.01 \pm 1.05$ &  $88.79 \pm 0.14$ \\
        0.002 &  $93.67 \pm 0.15$ &  $92.58 \pm 0.08$ &  $86.51 \pm 1.10$ &  $88.27 \pm 0.15$ \\
        0.003 &  $93.18 \pm 0.11$ &  $92.14 \pm 0.04$ &  $85.95 \pm 1.07$ &  $87.83 \pm 0.18$ \\
        0.005 &  $92.32 \pm 0.25$ &  $91.27 \pm 0.08$ &  $84.81 \pm 1.15$ &  $86.72 \pm 0.26$ \\
        0.007 &  $91.46 \pm 0.24$ &  $90.22 \pm 0.12$ &  $83.68 \pm 1.18$ &  $85.62 \pm 0.25$ \\
        0.010 &  $89.99 \pm 0.28$ &  $88.72 \pm 0.09$ &  $82.00 \pm 1.13$ &  $83.89 \pm 0.29$ \\
        0.100 &  $62.00 \pm 0.27$ &  $62.66 \pm 0.40$ &  $49.58 \pm 3.48$ &  $54.93 \pm 0.38$ \\
        0.200 &  $53.24 \pm 0.64$ &  $54.24 \pm 0.38$ &  $39.39 \pm 4.68$ &  $45.05 \pm 0.42$ \\
        0.300 &  $48.00 \pm 1.05$ &  $48.43 \pm 0.99$ &  $33.79 \pm 4.35$ &  $38.89 \pm 0.59$ \\
        0.400 &  $43.13 \pm 0.96$ &  $42.87 \pm 1.68$ &  $29.71 \pm 4.56$ &  $34.27 \pm 0.17$ \\
        0.500 &  $38.03 \pm 1.30$ &  $36.97 \pm 1.80$ &  $26.06 \pm 3.60$ &  $29.49 \pm 0.31$ \\
        1.000 &  $16.06 \pm 1.20$ &  $15.96 \pm 1.56$ &  $13.61 \pm 1.22$ &  $14.24 \pm 0.68$ \\
        1.200 &  $11.94 \pm 0.93$ &  $11.88 \pm 0.53$ &  $11.20 \pm 0.95$ &  $11.53 \pm 0.59$ \\
        1.400 &  $10.02 \pm 0.69$ &   $9.93 \pm 0.34$ &   $9.81 \pm 0.82$ &   $9.90 \pm 0.83$ \\
        1.500 &   $9.76 \pm 0.61$ &   $9.75 \pm 0.22$ &   $9.46 \pm 0.77$ &   $9.67 \pm 0.77$ \\
        \bottomrule
        \end{tabular}
\end{table}



%
%
\end{document}